\documentclass{article} %
\usepackage{dopt_preprint,times}

\usepackage{wrapfig}

\usepackage{amsmath,amsfonts,bm}

\usepackage{amsthm}
\usepackage{enumitem}
\usepackage{algorithm}
\usepackage{algpseudocode}
\usepackage{graphicx}
\usepackage{booktabs}
\usepackage{multirow}
\usepackage{adjustbox}
\usepackage{mdframed}
\usepackage{xcolor}

\newcommand{\eg}{{\it e.g.}}
\newcommand{\ie}{{\it i.e.}}
\newcommand{\st}{{\textrm{s.t.}}}

\newcommand{\dev}[1]{\mathrm{d} #1}

\newcommand{\vecm}[1]{\mathrm{vec} (#1)}
\newcommand{\ip}[1]{\left \langle #1 \right \rangle }

\newcommand{\aff}[1]{\text{aff}(#1)}

\newcommand{\coneK}{\mathcal{K}}
\newcommand{\relint}[1]{\text{relint}(#1)}
\algrenewcommand\algorithmicrequire{\textbf{Input:}}
\algrenewcommand\algorithmicensure{\textbf{Output:}}
\newcounter{problem}
\newcommand{\problabel}[1]{\stepcounter{equation}\refstepcounter{problem}\tag{P\theproblem}\label{#1}}

\newtheorem{theorem}{Theorem}[section]
\newtheorem{definition}{Definition}[section]
\newtheorem{proposition}{Lemma}[section]

\surroundwithmdframed[
  topline=false,
  bottomline=false,
  rightline=false,
  leftline=true,
  linewidth=3pt,
  linecolor=black!30,
  leftmargin=0.0em,
  innerleftmargin=0.7em,
  rightmargin=0em,
  innerrightmargin=0em,
  innertopmargin=0.15em,
  innerbottommargin=0.15em,
  skipabove=0.7em,
  skipbelow=0.7em
]{theorem}

\def\eqref#1{equation~\ref{#1}}
\def\Eqref#1{Equation~\ref{#1}}

\def\1{\bm{1}}

\DeclareMathAlphabet{\mathsfit}{\encodingdefault}{\sfdefault}{m}{sl}
\SetMathAlphabet{\mathsfit}{bold}{\encodingdefault}{\sfdefault}{bx}{n}

\DeclareMathOperator*{\argmin}{arg\,min}

\newcommand{\Deltaz}{\bar{ z }}
\newcommand{\DeltaZ}{\bar{ Z }}
\newcommand{\Jap}{J_{\mathrm{ap}}}
\newcommand{\Jna}{J_{\mathrm{na}}}

\numberwithin{table}{section}
\numberwithin{figure}{section}

\usepackage{hyperref}

\usepackage{url}
\usepackage{array}
\usepackage{placeins}

\hypersetup{hidelinks}

\usepackage{titlesec}

\title{{\upshape d}OPT: Differentiating Conic Optimization via Geometric Reduction}

\author{Fengyu Yang, Connor W. Magoon, Tyler Watts, Shahar Z. Kovalsky \\
Department of Mathematics \\
University of North Carolina at Chapel Hill
}

\begin{document}

\maketitle

\begin{abstract}

Optimization layers enable the incorporation of structured constraints and decision problems into learning systems. Training such systems requires differentiating through the embedded optimization problem, which can be challenging for general conic programs. We introduce \textbf{dOPT}, a solver-agnostic framework that, rather than differentiating the full conic formulation, reduces it at a computed primal-dual solution to an equality-constrained quadratic program that preserves the reference solution and its first-order sensitivity. The reduction captures the local first- and second-order conic geometry relevant to differentiation and remains well defined at singular configurations. Computing solution derivatives then requires a single symmetric linear solve, independently of the forward solver. We derive explicit reductions for convex NLPs, QPs, SOCPs, and SDPs. Numerical experiments validate the computed gradients and show favorable backward-pass scalability, with substantial speedups over existing differentiable conic optimization methods as problem size increases.

\end{abstract}

\section{Introduction}
\label{sec:introduction}

Optimization layers incorporate mathematical structure into learning models by defining their outputs through parameterized optimization problems \citep{amos2017optnet,cvxpylayers2019}. They have been used across control, decision-making, robotics, imaging, and other structured learning tasks. Training such models requires differentiating the optimal solution with respect to problem parameters, \ie, characterizing its local sensitivity to parameter perturbations, making the backward pass central to the practical use of optimization layers.

A standard approach differentiates the optimality conditions of the original problem. OptNet differentiates the KKT system for quadratic programs (QPs) \citep{amos2017optnet}, while DiffCP and CVXPYLayers extend implicit differentiation to conic programs \citep{agrawal2019diffcp,cvxpylayers2019}. However, the structure needed to \emph{solve} a problem often exceeds what is needed to characterize the \emph{local sensitivity} of its solution. Exploiting this distinction, dQP \citep{magoon2026differentiation} constructs locally first-order-equivalent equality-constrained problems for QPs, BPQP \citep{pan2024bpqp} formulates an equality-constrained backward QP for smooth convex programs, including certain SOCPs, and FFOLayer \citep{zhao2025fully} approximates gradients via finite differences on an equality-constrained surrogate. These methods use a solution computed by an arbitrary forward solver to construct a simpler proxy problem from which derivatives are computed.

This suggests a broader reduction principle: retain only the local constraint geometry that affects first-order sensitivity at the solution. Extending this principle beyond linearly constrained problems, however, is nontrivial. For linear inequality constraints, the relevant local geometry can be captured by a carefully selected subset of active constraints. More general, nonpolyhedral cones have curved boundaries whose normals vary under local perturbations of the solution. Simply replacing the cone locally by its tangent space can therefore be insufficient: changes in the active normal direction introduce a curvature term that is essential for recovering the correct solution derivative. The challenge extends to singular boundary configurations where a smooth local description may fail, including second-order cone apexes and PSD matrices with multiple zero eigenvalues.

We introduce \textbf{dOPT}, a solver-agnostic geometric framework that constructs an explicit, first-order equivalent equality-constrained QP from a primal--dual solution and the local conic geometry. Building on classical sensitivity analysis \citep{fiacco1983introduction,robinson1980strongly,bonnans2000perturbation}, dOPT identifies the perturbation directions relevant to the derivative of the solution map through the \emph{critical cone}, while accounting for variation of the cone normal along those directions. Under standard regularity conditions, these effects are encoded by two matrices: $C^*$, whose nullspace is the critical cone, and $H^*$, which captures the corresponding normal variation. This construction naturally accommodates the singular configurations described above.

Together, $C^*$ and $H^*$ define the constraints and curvature correction of the locally equivalent reduced formulation, which preserves both the reference solution and its parameter derivative. Differentiating its KKT conditions yields a single symmetric linear system, and the construction requires only the problem data and a computed primal--dual solution. The derivative can therefore be efficiently evaluated independently of how the forward solution is obtained, allowing mature black-box solvers such as Gurobi or MOSEK to be used without differentiating through their algorithms.

\textbf{Our main contributions are:}
\begin{itemize}[leftmargin=*, topsep=0em]
    \item We derive a first-order equivalent reduction for convex conic programs that preserves both the reference solution and its sensitivity, while naturally accommodating nonsmooth conic geometry.

    \item We provide explicit critical-cone and curvature constructions for NLPs, QPs, SOCPs, and SDPs, yielding reduced equality-constrained QPs from which solution derivatives can be efficiently computed with a single symmetric linear solve. In particular, the derivations accommodate singular configurations such as SOC apexes and PSD solutions with multiple zero eigenvalues.

    \item We provide a straightforward, solver-agnostic implementation of dOPT, with full flexibility in the choice of forward solver. The code will be publicly released upon acceptance.
    
    \item We evaluate our implementation on synthetic SOCPs and SDPs and an SDP-based learning task, validating gradient accuracy on the synthetic problems and demonstrating favorable backward-pass scalability and substantial speedups over existing differentiable conic optimization methods.

\end{itemize}

\section{Related Work}
\label{sec:related-work}

Differentiable optimization has diverse applications, including control \citep{Amos2018MPC}, differentiable physics \citep{Belbute-Peres2018DiffPHysforControl}, and certifiable robotics \citep{Holmes2024SDPCertifiableGlobalDerivative}. A similarly broad range of differentiation methods has been developed, including modular implicit differentiation \citep{blondel2022efficient,besanccon2024flexible}, conic differentiation \citep{agrawal2019diffcp,cvxpylayers2019,healey2025differentiating}, alternating and augmented-Lagrangian methods \citep{sun2022alternating,butler2023efficient,butler2023scqpth,bambade2024qplayer}, solver-iteration differentiation \citep{oshin2026unfolding,liu2026residual}, and penalty or perturbation methods \citep{linghu2026penalty,LPGD}. Extensions address nonconvexity \citep{rosemberg2025streamlined}, GPU computation \citep{moreau2026}, and code generation \citep{schaller2025code}. Our focus is sensitivity analysis and local reductions of the backward problem, which are most closely related to our approach.

Classical sensitivity analysis characterizes solution derivatives through constraint qualifications, critical cones, and second-order geometry \citep{fiacco1983introduction,robinson1980strongly,bonnans2000perturbation,bonnans1999second,mohammadi2021parabolic}, including nonpolyhedral curvature \citep{bonnans1998sensitivity} and SOCP/SDP sensitivity \citep{bonnans2005perturbation,shapiro1997first,alizadeh1997complementarity,sun2006strong}. Related work addresses partial smoothness and stratification \citep{lewis2002active,bao2026stratification}, weaker regularity \citep{andreani2023minimal,andreani2024weak,fukuda2023weak}, nonsmooth differentiation and nonunique solutions \citep{bolte2021nonsmooth,bolte2026adjoint}, composite bilevel problems \citep{solla2026optimistic}, and directional influence \citep{wang2026directional}. KKT-based implicit differentiation connects this theory to optimization layers \citep{barratt2018differentiability}; see \citet{pacaud2025sensitivity} for a recent survey.

Several methods simplify the backward problem through local reductions. BPQP \citep{pan2024bpqp} formulates an equality-constrained backward QP for smooth convex programs, including certain SOCPs, while dQP \citep{magoon2026differentiation} constructs a first-order-equivalent equality-constrained QP through active-set reduction. FFOLayer \citep{zhao2025fully} constructs an equality-constrained surrogate with linearized active constraints and a Lagrangian objective, then approximates hypergradients via finite differences using perturbed surrogate solves. Related reductions appear in SOCP trajectory sensitivity \citep{xu2025trajectory} and tangent-space methods for decision-focused learning \citep{lee2026pear}. dOPT extends this viewpoint to general convex conic programs, using critical-cone geometry and curvature to construct a reduced QP whose KKT system yields solution derivatives through a single symmetric linear solve.

\section{Preliminaries}
\label{sec:preliminaries}

In this paper, we consider the \textit{parametric convex conic program }
\begin{equation}
\begin{aligned}
z^*(\theta) = \argmin_{z} \quad & f_{\theta}(z)  \\
\st \quad & G_{\theta}(z) \in \coneK \\
 & h_{\theta}(z) = 0, \\
\end{aligned} \problabel{prob:conic}
\end{equation}
where $z\in\mathbb{R}^{n}$ is the optimization variable and $\theta\in\mathbb{R}^{s}$ is the problem parameter. The \textit{solution map} $z^*(\theta):\mathbb{R}^{s}\to\mathbb{R}^n$ maps parameters to optimal solutions. The parametric objective and constraint functions $f_{\theta}(z):=f(z,\theta):\mathbb{R}^n\times\mathbb{R}^{s}\to\mathbb{R}$, $G_{\theta}(z):=G(z,\theta):\mathbb{R}^n\times\mathbb{R}^{s}\to\mathbb{R}^m$ and $h_{\theta}(z):=h(z,\theta):\mathbb{R}^n\times\mathbb{R}^{s}\to\mathbb{R}^l$ and their Jacobians with respect to $z$ are continuously differentiable jointly in $(z,\theta)$. For fixed $\theta$, $f_\theta(z)$ is convex, $h_{\theta}(z)$ is affine and the set $\{z:G_{\theta}(z)\in \mathcal{K}\}$ is convex, where $\coneK$ is a closed convex cone. The parameter $\theta$ encodes perturbable problem data, and we study the response (sensitivity) of $z^*(\theta)$ to variations in this data. For example, in a quadratic program with $f_{\theta}(z)=\frac{1}{2}z^TPz+q^Tz$, $\theta$ may encode $(P,q)$, and the corresponding sensitivities include derivatives such as $dz^* / dP$ and $dz^* / dq$.

\subsection{The KKT Optimality Conditions and the Lagrangian}
The first-order Karush–Kuhn–Tucker (KKT) conditions \citep{Karush1939, Kuhn1951, boyd2004convex, nocedal2006numerical} characterize the optimal points of the convex constrained optimization problems, which are the key to analyzing the sensitivity of the optimal solution with respect to the parameters. The KKT conditions for the parametric convex conic program \ref{prob:conic} are
\begin{subequations}
\begin{align}
\nabla_z \mathcal{L}_{\theta}(z^*(\theta),\mu^*(\theta),\nu^*(\theta)) &= 0
\label{eqn:kkt_stationarity} \\
h_{\theta}(z^*(\theta)) &= 0
\label{eqn:kkt_primal_feas} \\
y^*(\theta) := G_{\theta}(z^*(\theta)) &\in \coneK
\label{eqn:kkt_primal_cone} \\
\mu^*(\theta) &\in \coneK^*
\label{eqn:kkt_dual_cone} \\
\langle \mu^*(\theta), y^*(\theta)\rangle &= 0
\label{eqn:kkt_complementarity}
\end{align}
\end{subequations}
where $\mu^*(\theta)$ and $\nu^*(\theta)$ are the \textit{optimal dual variables} associated with the conic and affine constraints, respectively, stationarity \Eqref{eqn:kkt_stationarity} is expressed in terms of the \textit{Lagrangian} $\mathcal{L}_{\theta}(z,\mu,\nu) = f_{\theta}(z) - \mu^T G_{\theta}(z) - \nu^T h_{\theta}(z)$,
and $\coneK^* = \{\mu : \langle \mu, y \rangle \geq 0, \forall y \in \coneK \}$ is the \textit{dual cone} of $\coneK$.

\subsection{Cone Geometry}\label{sec:conic_geometry}

To analyze how the solution of Problem~\ref{prob:conic} responds to changes in the problem parameters, we characterize the local cone geometry near a primal--dual solution. The \textit{tangent cone} describes feasible first-order directions at a point $y \in \coneK$, while the \textit{normal cone} describes the corresponding supporting directions:
$$
\mathcal{T}_{\coneK}(y) := \text{cl} \{ \alpha(u-y) : \alpha \geq 0, u \in \coneK \}, \qquad
\mathcal{N}_{\coneK}(y) := \{ w : \langle w, u-y \rangle \leq 0, \forall u \in \coneK \}.
$$
The same notions apply to the dual cone $\coneK^*$, giving $\mathcal{T}_{\coneK^*}(\mu)$ and $\mathcal{N}_{\coneK^*}(\mu)$ for $\mu \in \coneK^*$. Together, these cones characterize the feasible first-order primal--dual perturbation directions \citep{Rockafellar1970ConvexAnalysis,bonnans2000perturbation,boyd2004convex}.

The key geometric object governing solution sensitivity is the \textit{critical cone} associated with a complementary primal--dual pair $y^* \in \coneK$ and $\mu^* \in \coneK^*$,
$$
\mathcal{C}_{\coneK}(y^*) = \mathcal{T}_{\coneK}(y^*) \cap \mu^{*\perp}, \qquad
\mathcal{C}_{\coneK^*}(\mu^*) = \mathcal{T}_{\coneK^*}(\mu^*) \cap y^{*\perp},
$$
\begin{wrapfigure}[8]{r}{0.26\textwidth}
    \vspace{-1.5em}
    \includegraphics[width=\linewidth]{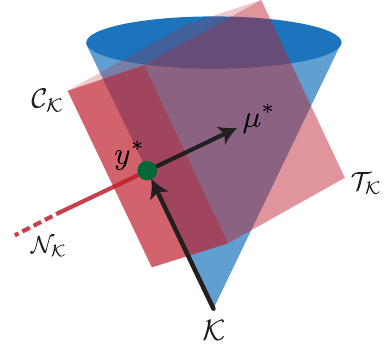}
    \label{fig:schematic}
\end{wrapfigure}
which is determined entirely by the local first-order cone geometry \citep[Proposition 3.10, p.~152]{bonnans2000perturbation}. The inset illustrates a complementary pair and its associated cones. While the tangent cones encode first-order feasibility, the critical cones further restrict perturbations through the linearized complementarity condition \Eqref{eqn:kkt_complementarity}. Intuitively, this restricts attention to the perturbation directions relevant to sensitivity analysis and along which cone curvature must be accounted for.
These are cone-space objects; linearized equality constraints are imposed separately for the full problem.

\subsection{Regularity Assumptions and Notation} \label{sec:notations}

Our analysis relies on standard assumptions from sensitivity analysis. Strict complementarity holds at a complementary primal--dual pair $(y^*,\mu^*)$ if $\mu^*\in\relint{-\mathcal{N}_{\coneK}(y^*)}$; geometrically, it provides a clear separation between active and inactive cone geometry and supports its stable identification under small perturbations. We additionally assume non-degeneracy, the second-order sufficient condition and $\mathcal{C}^{2}$-cone reducibility whose definitions are deferred to Appendix~\ref{app:preliminaries}. Under non-degeneracy and strict complementarity, the critical cone is in fact a linear subspace,
$\mathcal{C}_{\coneK}(y^*)=\aff{\mathcal{N}_{\coneK}(y^*)}^{\perp}$
\citep[Proposition 4.73, p.~317]{bonnans2000perturbation}. This property is central to constructing a reduced problem that preserves the first-order behavior of the original problem.

Throughout, $(\theta)$ denotes dependence of a solution map on the parameter $\theta$, whereas the subscript $\theta$ denotes a function parameterized by $\theta$. Let $(z^*(\theta),\mu^*(\theta),\nu^*(\theta))$ denote the primal--dual solution map. Fix a reference parameter $\theta_0$ and define the reference solution $(z^*,\mu^*,\nu^*):=(z^*(\theta_0),\mu^*(\theta_0),\nu^*(\theta_0))$. For gradients, Jacobians and Hessians evaluated at the reference solution, we suppress the solution arguments while retaining the parameter subscript; for example, the Hessian of the Lagrangian $\nabla^2_{zz}\mathcal{L}_{\theta}:=\nabla^2_{zz}\mathcal{L}_{\theta}(z^*,\mu^*,\nu^*)$. Evaluation is indicated explicitly by the subscript; for example, $\nabla_z G_{\theta}:= \nabla_z G_\theta(z^*)$, and $\nabla_z G_{\theta_0}:= \left.\nabla_z G_\theta(z^*)\right|_{\theta=\theta_0}$. For a direction $\dev\theta\in\mathbb R^s$, denote the directional differentials at $\theta_0$ by
$$
\dev z^*:=Dz^*(\theta_0)[\dev\theta],\qquad
\dev\mu^*:=D\mu^*(\theta_0)[\dev\theta],\qquad
\dev\nu^*:=D\nu^*(\theta_0)[\dev\theta].
$$
Similarly, define
$$
y^*(\theta):=G_\theta(z^*(\theta)),\qquad
y^*:=G_{\theta_0}(z^*),\qquad
\dev y^*:=Dy^*(\theta_0)[\dev\theta].
$$

\section{dOPT: First-Order Equivalent Reduction}
\label{sec:method}
This section presents our main contribution: a reduced problem that is first-order equivalent to the conic problem \ref{prob:conic} at the optimal solution $(z^*,\mu^*,\nu^*)$, yielding the same local solution and the same derivative with respect to the parameters. The central observation is that only the active cone geometries at the solution influence its first-order sensitivity. By restricting attention to these active directions and eliminating those that do not contribute to the gradient, we obtain a simplified problem. The resulting reduced formulation is a quadratic program with only equality constraints, capturing precisely the degrees of freedom relevant for differentiation. As a result, the gradient can be computed efficiently by solving a single linear system. This framework fully decouples optimization and differentiation: one may use any black-box solver to compute the solution of the original problem, and subsequently differentiate it solely through the reduced problem.

Under the nondegeneracy, strict complementarity, second-order sufficient conditions and $\mathcal{C}^{2}$-cone reducibility at the optimal solution $(z^*,\mu^*,\nu^*)$ of Problem~\ref{prob:conic}, the following theorem constructs a reduced problem that is first-order equivalent to the original conic problem. A proof is provided in Appendix~\ref{app:proof_conic_program}.

\begin{theorem}\label{theorem:conic_program}
The conic Problem~\ref{prob:conic} is first-order equivalent at $\theta_0$ to
\begin{equation}
\begin{aligned}
\hat{z}(\theta) = \argmin_{z} \quad & \nabla_{z}f_{\theta}^{T} \Deltaz + \frac{1}{2}\Deltaz^{T} \left(\nabla^{2}_{zz} \mathcal{L}_{\theta} - \nabla_z G_{\theta}^{T}H^* \nabla_z G_{\theta}\right) \Deltaz \\
\st \quad & C_\theta^*\left(G_\theta(z^*)-G_{\theta_0}(z^*) + \nabla_z G_{\theta}\Deltaz\right) = 0 \\
& h_\theta(z^*) + \nabla_z h_\theta\Deltaz = 0
\end{aligned}
\problabel{prob:reduced_conic}
\end{equation} %

where $\Deltaz:=z-z^*$. The matrices $C_\theta^*$ and $H^*$ form a compatible pair: $C_\theta^*$ varies smoothly, has full row rank, and satisfies $\ker C^*=\mathcal C_{\coneK}(y^*)$, where $C^*:=C_{\theta_0}^*$. The symmetric curvature matrix $H^*$ determines the dual variation up to a normal residual: $\mathrm d\mu^*-H^*\mathrm dy^*\in\mathcal C_{\coneK}(y^*)^\perp$ for every solution differential $(\mathrm dy^*,\mathrm d\mu^*)$.

At $\theta_0$, the primal--dual solution map $(\hat z(\theta),\hat\mu(\theta),\hat\nu(\theta))$ of the reduced Problem~\ref{prob:reduced_conic} is related to the primal--dual solution map of the original problem by
\begin{equation}
    z^* = \hat{z}(\theta_0), \qquad
    \mu^* = C^{*T}\hat{\mu}(\theta_0), \qquad
    \nu^* = \hat{\nu}(\theta_0),
\end{equation}
and their directional differentials satisfy, for every parameter direction $\dev\theta$,
\begin{equation}
    \dev z^* = \dev\hat{z}, \qquad
    \dev\mu^* = H^* \dev y^* + C^{*T}\dev\hat{\mu}, \qquad
    \dev\nu^* = \dev\hat{\nu}.
\label{eqn:thm_dmu}
\end{equation}
Consequently, the differential $(\dev \hat{z}, \dev \hat{\mu}, \dev \hat{\nu})$ is obtained by solving the linear system
\begin{equation}
\begin{bmatrix}
Q_{\theta_0} & -\nabla_z G_{\theta_0}^T C^{*T} & -\nabla_z h_{\theta_0}^T \\
-C^*\nabla_z G_{\theta_0} & 0 & 0\\
-\nabla_z h_{\theta_0} & 0 & 0
\end{bmatrix}
\begin{bmatrix}
\dev \hat{z} \\
\dev \hat{\mu} \\
\dev \hat{\nu}
\end{bmatrix}
=
\begin{bmatrix}
-\rho_{\theta_0} \\
C^*\nabla_\theta G_{\theta_0} \\
\nabla_\theta h_{\theta_0}
\end{bmatrix}\dev\theta,
\label{eqn:reduced_conic_diff_kkt_matrix}
\end{equation}
where
\begin{equation*}
\begin{aligned}
Q_{\theta_0} &:= \nabla^{2}_{zz} \mathcal{L}_{\theta_0} - \nabla_z G_{\theta_0}^T H^* \nabla_z G_{\theta_0},\\
\rho_{\theta_0} &:= \nabla^{2}_{z\theta} \mathcal{L}_{\theta_0}-\nabla_z G_{\theta_0}^T H^* \nabla_\theta G_{\theta_0}.
\end{aligned}
\end{equation*}
\end{theorem}

Theorem~\ref{theorem:conic_program} yields a solver-agnostic procedure for computing solution derivatives. Given a primal--dual solution $(z^*,\mu^*,\nu^*)$ obtained from any forward solver, the matrices $C^*$ and $H^*$ are constructed from the local cone geometry at $y^*$. These characterize the critical cone $\mathcal{C}_{\coneK}(y^*)$ and the variation of the dual solution along it, as in \Eqref{eqn:thm_dmu}. They define the reduced problem~\ref{prob:reduced_conic}, from which $\mathrm{d}z^*$, $\mathrm{d}\mu^*$, and $\mathrm{d}\nu^*$ are recovered by solving the linear system~\Eqref{eqn:reduced_conic_diff_kkt_matrix}.

\begin{wrapfigure}[15]{r}{0.18\textwidth}
    \centering
    \vspace{-2.2em}
    \includegraphics[width=0.17\textwidth]{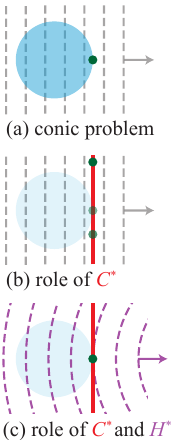}
\end{wrapfigure}
\textbf{Remarks.}
(\textit{i}) $C^*$ and $H^*$ capture complementary first- and second-order components of the local cone geometry, as illustrated in the inset. Panel (a) depicts the original conic problem. In (b), $C^*$ restricts perturbations to the admissible directions in the critical cone, but this first-order linearization alone may admit extraneous solutions that violate the original conic constraint. In (c), the quadratic correction induced by $H^*$ accounts for the variation of the cone normal along these directions, removing this ambiguity. Crucially, this correction preserves the parameter sensitivity of the original problem, making the reduced problem first-order equivalent as stated in Theorem~\ref{theorem:conic_program}.
\smallskip
\\
(\textit{ii}) Global smoothness of $\mathcal{K}$ is not required. When the active boundary is locally smooth, the restriction of $H^*$ acts, up to sign convention, as the scaled shape operator, and
$-\dev y^{*T}H^*\dev y^*$ is the associated curvature correction, closely related to the second fundamental form \citep{doCarmo1976}. At conic singularities, such as an SOC apex or PSD matrices with multiple zero eigenvalues, a smooth shape operator is undefined \citep{fukuda2023weak}. Our framework instead adopts the variational perspective of \citet{bonnans2000perturbation}, where the solution differential is restricted to the critical cone,
$\dev y^*\in\mathcal{C}_{\coneK}(y^*)$. Curvature therefore need only be characterized along these admissible directions, where it can be expressed through the support function of the corresponding second-order tangent set. This directional curvature admits the quadratic representation induced by $H^*$, while $C^*$ excludes singular directions along which the curvature is not well-defined (see Sections~\ref{sec:socp} and~\ref{sec:sdp}).
\smallskip
\\
(\textit{iii}) In practical applications, \eg, neural-network training, one need not solve \Eqref{eqn:reduced_conic_diff_kkt_matrix} explicitly to backpropagate gradients; instead, the required vector--Jacobian product can be computed by solving the adjoint system with the same symmetric coefficient matrix. The framework may also extend to nonconvex problems with locally convex geometry under suitable regularity conditions, although this extension is beyond the scope of this paper.

\section{Applications to NLP, QP, SOCP, and SDP}
\label{sec:instances}

To demonstrate the utility of Theorem~\ref{theorem:conic_program}, we explicitly instantiate the framework to derive reduced formulations for nonlinear programs (NLPs), quadratic programs (QPs) as a special case of NLPs, second-order cone programs (SOCPs), and semidefinite programs (SDPs). Theorem~\ref{theorem:instance_equivalence} below summarizes the resulting formulations for all these classes. Sections~\ref{sec:nlp}--\ref{sec:sdp} then provide the corresponding derivations, constructing the matrices $H^*$ and $C^*$ and capturing the relevant local cone geometry, including nonsmooth configurations at SOC apex blocks and SDP solutions with multiple zero eigenvalues.

\begin{theorem}%
\label{theorem:instance_equivalence}
Under the standing regularity assumptions of Section~\ref{sec:notations}, consider any of the NLP, QP, SOCP, or SDP problems in the table below. The reduced problem in each corresponding row is first-order equivalent to the original problem at the reference parameter $\theta_0$.

\newcommand{\smat}[2]{\bigl[\begin{smallmatrix}#1\\#2\end{smallmatrix}\bigr]}
\newcommand{\instancecell}[1]{%
  \raisebox{0pt}[\dimexpr\height+1.3ex\relax][\dimexpr\depth+1.3ex\relax]{#1}%
}
\newcommand{\instanceheader}[1]{%
  \raisebox{0pt}[\dimexpr\height+0.5ex\relax][\dimexpr\depth+0.5ex\relax]{#1}%
}

\newcommand{\conicCell}[1]{%
  \raisebox{0pt}%
    [\dimexpr\height+0.8ex\relax]%
    [\dimexpr\depth+0.8ex\relax]{#1}%
}

\newcommand{\conicLabel}[1]{%
  \makebox[\linewidth][c]{%
    \rotatebox[origin=c]{90}{\textnormal{\textbf{#1}}}%
  }%
}

\newcommand{\conicArray}[1]{%
  \begingroup
  \renewcommand{\arraystretch}{1.25}%
  \begin{adjustbox}{max width=\dimexpr\linewidth-3pt\relax}
  $\begin{array}{>{\displaystyle}r@{\ }>{\displaystyle}l}
  #1
  \end{array}$%
  \end{adjustbox}%
  \endgroup
}

\newcommand{\conicHeader}[1]{%
  \raisebox{0pt}%
    [\dimexpr\height+0.8ex\relax]%
    [\dimexpr\depth+0.8ex\relax]{\textnormal{\textbf{#1}}}%
}

\begin{adjustbox}{max width=\textwidth}
\begin{tabular}{|m{0.2cm}|m{4.5cm}|m{8.6cm}|}
\hline

&
\conicHeader{Original Problem} &
\conicHeader{Reduced Problem} \\
\hline

\conicCell{\conicLabel{NLP [\ref{sec:nlp}]}}
&
\conicCell{%
\conicArray{
\min_z \ \ \
  & f(z)\\
\st \ \ \
  & g_{j}(z)\le 0
}
}
&
\conicCell{%
\conicArray{
\min_z \ \ \
  & \nabla_z f^T z
    +\tfrac12\Deltaz ^TQ_{\mathrm{NLP}}\Deltaz \\
\st \ \ \
  & g_{j}(z^*)
    +\nabla_z g_{j}^T\Deltaz =0,
    \quad j\in J\\[2ex]
\text{with} \ \ \
  & Q_{\mathrm{NLP}}
    =\nabla_{zz}^2 f
    -\sum_{j\in J}
      \mu^{*}_{j}\nabla_{zz}^2g_{j}
}
}
\\
\hline

\conicCell{\conicLabel{QP [\ref{sec:qp}]}}
&
\conicCell{%
\conicArray{
\min_z \ \ \
  & q^T z + \frac{1}{2} z^T P z\\
\st \ \ \
  & c_j^T z \leq d_j
}
}
&
\conicCell{%
\conicArray{
\min_z \ \ \
  &  q^T z + \frac{1}{2} z^T P z \\
\st \ \ \
  & c_j^T z = d_j, \quad j\in J\\%
}
}
\\
\hline

\conicCell{\conicLabel{SOCP [\ref{sec:socp}]}}
&
\conicCell{%
\conicArray{
\min_z \ \ \
  & q^Tz\\
\st \ \ \
  & \|A_jz+b_j\|\le c_j^Tz+d_j
}
}
&
\conicCell{%
\conicArray{
\min_z \ \ \
  & q^Tz
    +\tfrac12\Deltaz ^TQ_{\mathrm{SOCP}}\Deltaz \\
\st \ \ \
  & \begin{alignedat}[t]{2}
    & \nabla \phi(G_j(z^*))^T (G_j(z)-y^*_j) = 0,
      &\quad & j\in \Jna,\\[.1ex]
    & G_j(z) = 0,
      &\quad & j\in \Jap
    \end{alignedat} \\[4ex]
\text{with} \ \ \
  & Q_{\mathrm{SOCP}}
    =\sum_{j\in \Jna}s_j^*
      \nabla G_j ^T
      \nabla^2\phi(y^*_j)
      \nabla G_j \\[3ex]
    & G_j(z):=\begin{bmatrix}A_j\\c_j^T\end{bmatrix}z
       +\begin{bmatrix}b_j\\d_j\end{bmatrix}, \quad \phi(x,t):=\|x\|_2-t
}

}
\\
\hline

\conicCell{\conicLabel{SDP [\ref{sec:sdp}]}}
&
\conicCell{%
\conicArray{
\min_{Z=Z^T}
  & \ip{C,Z}\\
\st \ \ \
  & \ip{A_i,Z}=B_i,\\
  & Z\succeq 0
}
}
&
\conicCell{%
\conicArray{
\min_{Z=Z^T}
  & \ip{C,Z}
    +\tfrac12
      \vecm{\DeltaZ}^{\!T}
      Q_{\mathrm{SDP}}
      \vecm{\DeltaZ}\\
\st \ \ \
  & \ip{A_i,Z}=B_i\\[.2ex]
  & \ip{u_j^*u_{j'}^{*T},Z}=0,
    \quad j,j'\in J,\ j'\le j\\[2ex]
\text{with} \ \ \
  & Q_{\mathrm{SDP}}
    =Z^{*\dagger}\otimes S^*
     +S^*\otimes Z^{*\dagger}
}

}
\\
\hline

\end{tabular}
\end{adjustbox}

where $\Deltaz :=z-z^*$ and $\DeltaZ :=Z-Z^*$. For NLP and QP, $j$ indexes the inequality constraints and $J$ is the active subset. For SOCP, $j$ indexes the SOC blocks, while $\Jna $ and $\Jap $ denote the non-apex and apex
active subsets, respectively. For SDP, $i$ indexes the affine constraints, while $\{u_j^*\}_{j\in J}$ is an orthonormal basis for $\ker(Z^*)$. For compactness, we suppress the dependence of data on $\theta$.

\end{theorem}

\subsection{Nonlinear Programs}
\label{sec:nlp}
Consider the NLP in Theorem~\ref{theorem:instance_equivalence}, where $f_{\theta}(z), g_{j,\theta}(z): \mathbb{R}^{n} \times \mathbb{R}^{s}\rightarrow \mathbb{R}$ and their Jacobians are jointly $\mathcal{C}^1$ in $(z,\theta)$. The inequality constraints can be represented by the non-positive orthant $\mathcal{K} = \mathbb{R}_-^m$. Denote by $J:=\{j:y^*_{j} = g_{j,\theta_0}(z^*)=0\}$ the active set at the reference solution. Its critical cone is given by $
\mathcal C_{\mathcal K}(y^*)
=
\{\mathrm dy\in\mathbb R^m:\mathrm dy_j=0,\ j\in J\}
$, implying that perturbations must keep active constraints tight to zero. Consequently, $C^* = C_\theta^*$ is the coordinate-selection matrix associated with $J$. Since the boundary of the nonpositive orthant is locally polyhedral and hence has no curvature, $H^*=0$. Substituting these expressions into
Theorem~\ref{theorem:conic_program} yields the reduced NLP in Theorem~\ref{theorem:instance_equivalence}, recovering the classical active-set reduction for nonlinear programs \citep{fiacco1983introduction}.

\subsection{Quadratic Programming}
\label{sec:qp}

The case of NLP with $f_{\theta}(z) = q_{\theta}^T z + \frac{1}{2} z^T P_{\theta} z$ and $g_{j,\theta}(z) = c_{j,\theta}^T z - d_{j,\theta}$ readily recovers the reduced problem for the case of QP derived in  \citet{magoon2026differentiation}.

\subsection{Second-order Cone Programs}
\label{sec:socp}

The $j$-th second-order cone (SOC) constraint in Theorem~\ref{theorem:instance_equivalence} takes the form
\[
y_j\in\mathcal K_{\mathrm{SOC}_j}:=\{(x,t):\phi(x,t)\leq0\},
\]
where $y_j=G_j(z):=(A_jz+b_j,c_j^Tz+d_j)$ and $\phi(x,t):=\|x\|_2-t$. Recall that $y_j^*$ is the fixed value of $G_j(z^*)$ at the reference problem data.
Let $\mu_j=(u_j,s_j)$ denote the corresponding dual variable. Under our assumptions, each SOC block at the solution is either inactive, with $y_j^*\in\operatorname{int}\mathcal K_{\mathrm{SOC}_j}$, or active, with non-apex and apex index sets
\[
\begin{aligned}
\Jna&:=\{j:y_j^*\in\partial\mathcal K_{\mathrm{SOC}_j}\setminus\{0\},
\ \mu_j^*\in\partial\mathcal K_{\mathrm{SOC}_j}^*\setminus\{0\}\},\\
\Jap&:=\{j:y_j^*=0,
\ \mu_j^*\in\operatorname{int}(\mathcal K_{\mathrm{SOC}_j}^*)\}.
\end{aligned}
\]
The critical cones for active blocks are
\begin{equation}
\mathcal C_{\mathcal K_{\mathrm{SOC}_j}}(y_j^*)
=
\begin{cases}
\{\mathrm dy:\nabla\phi(y_j^*)^T\mathrm dy=0\}, & j\in\Jna,\\
\{0\}, & j\in\Jap.
\end{cases}
\label{eq:socp_critical_cone}
\end{equation}
Thus, first-order solution perturbations lie in the tangent hyperplane for non-apex blocks and vanish for apex blocks. These restrictions and the cone's local curvature determine $C_{j,\theta}^*$ and $H_j^*$:
\begin{proposition}\label{lemma:socp_H_C}
For each active SOC block,
\[
(H_j^*,C_{j,\theta}^*)=
\begin{cases}
\left(-s_j^*\nabla^2\phi(y_j^*),
\nabla\phi(G_j(z^*))^T\right), & j\in\Jna,\\
(0,I), & j\in\Jap.
\end{cases}
\]
\end{proposition}
The derivation appears in Appendix~\ref{app:socp}. Substituting these expressions into Theorem~\ref{theorem:conic_program} yields the reduced SOCP in Theorem~\ref{theorem:instance_equivalence}.

\textbf{Remark.} Beyond giving a closed-form reduction, this specialization illustrates the
local smoothness requirement underlying Theorem~\ref{theorem:conic_program}.
The cone boundary need not admit a smooth representation in every ambient
direction; its curvature is required only along admissible directions in the
critical cone. Although an SOC constraint can be written as the scalar inequality $\phi(y_j)\leq0$, this representation is nonsmooth at the apex $y_j^*=0$, so direct differentiation as a smooth NLP is not applicable at such configurations. By~\Eqref{eq:socp_critical_cone}, however, the critical
cone at the apex is $\{0\}$, so there is no nonzero admissible direction along
which curvature must be evaluated, resulting in a well-defined reduction. %

\subsection{Semidefinite Programs}
\label{sec:sdp}
Let $Z\in\mathbb S^n$ and $S\in\mathbb S^n$ be the primal and dual matrices associated with the constraint $Z\succeq0$ of the SDP in Theorem~\ref{theorem:instance_equivalence}.
To describe the local geometry of the PSD cone, we write the compact eigendecomposition $Z^* = U_1^*\Lambda_1^*U_1^{*T}$ with $\Lambda_1^*\succ0$, and let $U_0^*$ be an orthonormal basis for $\ker(Z^*)$ with columns $u_j^*$ indexed by $j \in J$. Under our assumptions, the critical cone is
\begin{equation}
\mathcal C_{\mathbb S_+^n}(Z^*)
=\{\mathrm dZ\in\mathbb S^n:U_0^{*T}\mathrm dZ\,U_0^*=0\}
=\{\mathrm dZ:u_j^{*T}\mathrm dZ\,u_{j'}^*=0,\ j,j'\in J\}.
\label{eqn:sdp_critical_cone}
\end{equation}
Thus, admissible first-order perturbations cannot modify the block of $Z^*$
supported on its zero-eigenspace. This restriction determines $C^* = C_\theta^*$, while
the curvature of the PSD cone along the remaining admissible directions
determines $H^*$ as follows:

\begin{proposition}\label{lemma:sdp_H_C}
The differential of the dual matrix admits the vector-form
decomposition
\[
\vecm{\mathrm dS^*}=H^*\vecm{\mathrm dZ^*}
+C^{*T}\mathrm d\hat\mu,
\qquad
H^*=-\left(Z^{*\dagger}\otimes S^*+S^*\otimes Z^{*\dagger}\right),
\]
where the rows of $C^*$ are $\vecm{u_j^*u_j^{*T}}^T$ and $\vecm{u_j^*u_{j'}^{*T}+u_{j'}^*u_j^{*T}}^T$ for $j,j'\in J, j'<j$.
\end{proposition}
A detailed derivation is deferred to Appendix~\ref{app:sdp}. Substituting these expressions into Theorem~\ref{theorem:conic_program} yields the reduced SDP in
Theorem~\ref{theorem:instance_equivalence}.

\textbf{Remark.} This instance further illustrates that differentiability does not require cone curvature to be well defined in all directions, but only along admissible directions in the critical cone. In the SDP case, nonsmoothness arises when perturbations change the zero-eigenspace block of $Z^*$. The critical-cone condition $U_0^{*T}\dev Z^* U_0^*=0$ removes exactly those directions; see Appendix~\ref{app:sdp} for an illustrative example. Therefore, the reduced problem remains well defined even when $Z^*$ has multiple zero eigenvalues.

\section{Experiments}
\label{sec:experiments}
We evaluate dOPT through two experiments. First, we assess gradient accuracy and computational scalability on synthetic SOCP and SDP instances. Second, we demonstrate dOPT in an end-to-end learning benchmark with a differentiable SDP layer, focusing on backward computational cost during training.
As baselines, we use CVXPYLayers (with diffCP as its backend)~\citep{cvxpylayers2019,agrawal2019diffcp} and FFOLayer~\citep{zhao2025fully}. To our knowledge, these are the closest publicly available differentiation frameworks applicable to general conic programs and supporting evaluation on both SOCP and SDP instances\footnote{We excluded Moreau~\citep{moreau2026} due to SDP failure, DiffOpt.jl~\citep{besanccon2024flexible} for lack of native PyTorch integration, and diffQCP~\citep{healey2025differentiating} due to inconsistent gradients across solvers.}.
In our experiments, both dOPT and FFOLayer solve the forward problems using MOSEK~\citep{andersen2000mosek}, while CVXPYLayers uses SCS~\citep{scs}. Our comparisons focus primarily on backward computational cost.

\subsection{Synthetic Evaluation}\label{sec:synthetic_experiments}

\begin{table}[tbp]
\begin{center}
\caption{Synthetic SOCP and SDP performance and scalability.}
\vspace{.2em}
\label{tab:random_conic}
\setlength{\tabcolsep}{4pt}
\renewcommand{\arraystretch}{0.96}
\resizebox{\textwidth}{!}{%
\begin{tabular}{llcccccccc}
\toprule
Method & Metric & \multicolumn{4}{c}{\textbf{SOCP}} & \multicolumn{4}{c}{\textbf{SDP}} \\
\cmidrule(lr){3-6}\cmidrule(lr){7-10}
& Problem size & 20 & 100 & 500 & 700 & $20\times20$ & $50\times50$ & $100\times100$ & $150\times150$ \\
\midrule
\multirow{4}{*}{dOPT (MOSEK)}
& Accuracy & $\mathbf{2.14\times10^{-8}}$ & $1.28\times10^{-7}$ & $3.06\times10^{-6}$ & $\mathbf{2.26\times10^{-5}}$ & $\mathbf{1.52\times10^{-8}}$ & $\mathbf{5.16\times10^{-9}}$ & $\mathbf{2.27\times10^{-9}}$ & $\mathbf{7.32\times10^{-9}}$ \\
& Forward [ms] & 14.3 & 260.1 & 64055.1 & \textbf{346639.9} & \textbf{24.5} & \textbf{615.4} & 15040.9 & \textbf{142613.1} \\
& Backward [ms] & \textbf{1.6} & \textbf{7.2} & \textbf{184.0} & \textbf{1442.8} & \textbf{3.4} & \textbf{193.2} & \textbf{6946.4} & \textbf{82890.7} \\
& Total [ms] & 15.9 & \textbf{267.3} & \textbf{64239.1} & \textbf{348082.7} & \textbf{27.9} & \textbf{808.6} & \textbf{21987.3} & \textbf{225503.8} \\
\midrule
\multirow{4}{*}{CVXPYLayers (SCS)}
& Accuracy & $2.15\times10^{-8}$ & $\mathbf{3.04\times10^{-8}}$ & $\mathbf{1.15\times10^{-6}}$ & Failed & $1.80\times10^{-6}$ & $2.57\times10^{-6}$ & $2.58\times10^{-6}$ & Failed \\
& Forward [ms] & \textbf{2.0} & \textbf{152.7} & 160960.3 & -- & 29.4 & 934.6 & 39438.4 & -- \\
& Backward [ms] & 1.9 & 773.0 & 83147.4 & -- & 69.2 & 6761.2 & 115012.5 & -- \\
& Total [ms] & \textbf{3.9} & 925.7 & 244107.7 & -- & 98.6 & 7695.8 & 154450.9 & -- \\
\midrule
\multirow{4}{*}{FFOLayer (MOSEK)}
& Accuracy & $\mathbf{2.14\times10^{-8}}$ & $1.22\times10^{-7\,*}$ & $2.00\times10^{-6\,*}$ & Failed & $\mathbf{1.52\times10^{-8}}$ & $\mathbf{5.16\times10^{-9}}$ & $\mathbf{2.27\times10^{-9}}$ & $\mathbf{7.32\times10^{-9}}$ \\
& Forward [ms] & 16.4 & 268.6 & \textbf{63283.2} & -- & 25.4 & 638.9 & \textbf{14551.6} & 148793.2 \\
& Backward [ms] & 31.3 & 334.8 & 46807.9 & -- & 25.6 & 689.4 & 15968.4 & 199709.1 \\
& Total [ms] & 47.7 & 603.5 & 110091.2 & -- & 51.1 & 1328.4 & 30520.0 & 348502.3 \\
\bottomrule
\end{tabular}%
}
\end{center}
\vspace{0pt}
{\scriptsize $^*$ Some FFOLayer backward passes errored; results are averaged over successful runs: 9/10 instances at $n=100$ and 6/10 at $n=500$.}
\end{table}

We evaluate dOPT on randomly generated SOCP and SDP instances of increasing size, averaging over ten instances per setting. We report forward and backward runtimes and forward-solution accuracy, measured by the duality gap. Problem-generation and implementation details are provided in Appendix~\ref{app:experiment_details}. Gradient validation in Appendix~\ref{app:gradient_accuracy} shows that optimal-value gradients computed by dOPT agree with analytic gradients from the envelope theorem \citep{afriat1971theory} to relative errors of approximately $10^{-9}$ for SOCPs and $10^{-7}$ for SDPs.

Table~\ref{tab:random_conic} summarizes performance and scalability. dOPT has the lowest backward runtime at every tested size. For SOCPs at $n=500$, it is over two orders of magnitude faster than both baselines. At $n=700$, only dOPT completes the benchmark; both baselines terminate with errors on nearly every instance. FFOLayer also has backward-pass failures at smaller sizes, as noted in the table. For SDPs at $n=100$, dOPT's backward pass is $2.3\times$ faster than FFOLayer and $16.6\times$ faster than CVXPYLayers; at $n=150$, CVXPYLayers terminates with errors on every instance.

\subsection{Learning Stable Switched Linear Systems}
\label{sec:switched_system}

To demonstrate dOPT in an end-to-end learning benchmark with a differentiable SDP layer, we consider a discrete-time switched linear system:
\begin{equation}
    x_{k+1}=A_{\sigma(k)}x_k,
\end{equation}
where $A_1,\ldots,A_M\in\mathbb{R}^{d\times d}$ are mode matrices and $\sigma(k)$ selects the active mode. Such systems model dynamics that switch between operating regimes. Given observed trajectories, the goal is to learn both the mode matrices $A_i$ and the switching signal $\sigma(k)$. Minimizing prediction error alone, however, does not ensure that the learned dynamics are stable. To promote stability, we regularize the training loss using a standard common quadratic Lyapunov certificate~\citep{ahmadi2014joint}. Specifically, a differentiable SDP layer computes
\begin{equation}
    t^*\bigl((A_i)_{i=1}^M\bigr)=\max_{P,t}\;t
    \quad\text{s.t.}\quad
    P\succeq0,\quad \operatorname{tr}(P)=1,\quad
    P-A_i^\top P A_i\succeq tI,\quad i=1,\ldots,M.
    \label{eq:switch_system}
\end{equation}
A positive margin $t^*>0$ certifies stability under arbitrary switching. We use the scaled margin $\alpha t^*$ to regularize trajectory fitting, giving
$\mathcal{L}=\mathcal{L}_{\mathrm{traj}}-\alpha t^*$.

Figure~\ref{fig:learning_experiments}(left) evaluates scalability on synthetic switched systems. Since the original FFOLayer implementation does not directly support parametric PSD constraints as in \Eqref{eq:switch_system}, we compare against a minimally modified version, \emph{FFOLayer (Modified)}, and an equivalent lifted formulation, \emph{FFOLayer (Lifted)}. At $d=40$, dOPT's backward pass is $6.4\times$ faster than the next-fastest FFOLayer variant and $97.9\times$ faster than CVXPYLayers; at $d=70$, it remains over an order of magnitude faster than FFOLayer and over two orders of magnitude faster than CVXPYLayers.

We also adapt the networked vehicle-platoon benchmark of \citet{makhlouf2014networked,chen2015benchmark}, with $N\in\{3,5,10\}$ vehicles. Each vehicle regulates a three-dimensional physical state using information from the others, with modes corresponding to normal communication and communication loss. This serves as a controlled benchmark for differentiation through the SDP layer rather than a real-world platooning solution. All methods learn nearly identical positive stability margins, while dOPT is fastest across all three sizes, completing training about $1.8$--$1.9\times$ faster than the next-fastest method. Figure~\ref{fig:learning_experiments}(right) shows the 10-vehicle training dynamics, and Table~\ref{tab:platoon_stability_compact} reports results for all three sizes; additional details appear in Appendix~\ref{app:platoon_details}.

\begin{figure}[tbp]
    \centering
    \includegraphics[width=\linewidth]{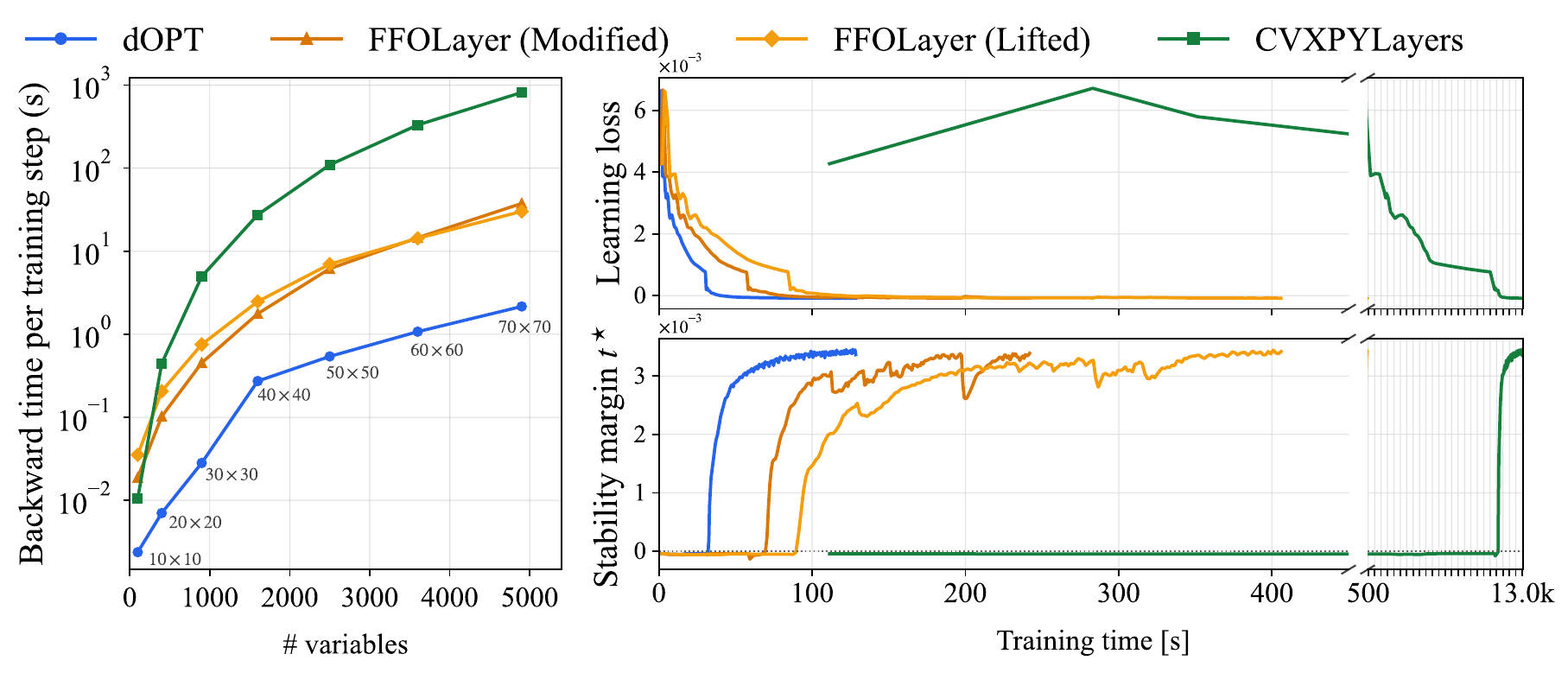}
    \vspace{-2em}
    \caption{Learning stable switched linear systems. Left: backward time against number of optimization variables $O(d^2)$ for synthetic switched linear systems. Right:
    learning loss and stability margin $t^*$ versus training time for the
    10-vehicle platoon system ($N=10$).}
    \label{fig:learning_experiments}
    \vspace{-1em}
\end{figure}

\begin{table}[tbp]
    \caption{Performance of the adapted vehicle-platoon benchmark for
    3, 5, and 10 vehicles.}
    \label{tab:platoon_stability_compact}
    \vspace{.4em}
    \begingroup
\centering
\scriptsize
\setlength{\tabcolsep}{2pt}
\renewcommand{\arraystretch}{0.95}

\begin{adjustbox}{max width=\textwidth}
\begin{tabular}{
    @{}l
    ccc
    @{}p{14pt}@{}
    ccc
    @{}p{14pt}@{}
    ccc
    @{}
}
\toprule
&
\multicolumn{3}{c}{\textbf{$9 \times 9$} $(N=3)$}
&
&
\multicolumn{3}{c}{\textbf{$15 \times 15$} $(N=5)$}
&
&
\multicolumn{3}{c}{\textbf{$30 \times 30$} $(N=10)$}
\\
\cmidrule(lr){2-4}
\cmidrule(lr){6-8}
\cmidrule(lr){10-12}

Method
& Fwd \tiny{[ms]} & Bwd \tiny{[ms]}  & Train \tiny{[s]}
&
& Fwd \tiny{[ms]} & Bwd \tiny{[ms]}  & Train \tiny{[s]}
&
& Fwd \tiny{[ms]} & Bwd \tiny{[ms]}  & Train \tiny{[s]}
\\
\midrule

dOPT
& $\mathbf{12.7}$ & $\mathbf{1.4}$ & $\mathbf{4.3}$
&
& $41.4$ & $\mathbf{3.6}$ & $\mathbf{12.0}$
&
& $468.4$ & $\mathbf{26.1}$ & $\mathbf{128.8}$
\\

FFOLayer (Modified)
& $14.3$ & $16.2$ & $8.0$
&
& $\mathbf{41.0}$ & $43.6$ & $21.4$
&
& $\mathbf{459.0}$ & $472.7$ & $241.9$
\\

FFOLayer (Lifted)
& $28.1$ & $30.5$ & $15.1$
&
& $84.3$ & $88.9$ & $44.2$
&
& $797.4$ & $815.1$ & $405.9$
\\

CVXPYLayers
& $93.3$ & $6.6$ & $28.6$
&
& $345.3$ & $84.3$ & $115.2$
&
& $5592.1$ & $4948.6$ & $13047.9$
\\

\bottomrule
\end{tabular}
\end{adjustbox}

\par
\endgroup

\end{table}

\section{Conclusion}

We introduced dOPT, a framework for differentiating through parametric convex conic programs via first-order equivalent equality-constrained QPs. Explicit reductions for NLPs, QPs, SOCPs, and SDPs accommodate geometrically singular configurations under standard regularity assumptions and enable efficient differentiation independently of the forward solver. Notable limitations include the restriction to convex problems, reliance on high-quality solutions, and limited exploitation of algebraic structure, especially for SDPs. Forward solves remain the main computational bottleneck, which dOPT does not address. These limitations motivate future work on extensions to nonconvex problems, support for approximate solutions (\eg, from L2O~\citep{chen2022learning}), more effective use of structure, weaker regularity assumptions, and other spectral cones.

\FloatBarrier

\bibliography{references}
\bibliographystyle{dopt_references}

\newpage
\appendix
\section{Preliminaries}
\label{app:preliminaries}
This section introduces additional definitions concerning cone geometry and regularity assumptions \citep{boyd2004convex,bonnans2000perturbation,Rockefellar1998VariationalAnalysis}. $\theta$ is omitted for simplicity.

\begin{definition}[Face]

A face of a convex set $C$ is a convex subset $F \subseteq C$ such that, whenever a strict convex combination of two points in $C$ lies in $F$, both points lie in $F$. That is, if $\theta x + (1-\theta)y \in F$ for some $x,y \in C$ and $0<\theta<1$, then $x,y \in F$.
\end{definition}

\begin{definition}[Active face]
The \emph{active face} of a convex cone $\coneK \subseteq \mathbb{R}^{n}$ at a point $y \in \coneK$ is defined as
\begin{equation*}
    \mathcal{F}_{\coneK}(y):=\bigcap\{F \subset \coneK: F\text { is a face of } \coneK, y \in F \}
\end{equation*}
\end{definition}

\begin{definition}[Non-degeneracy]
The problem \ref{prob:conic} is \emph{non-degenerate} at the primal solution $z^*$ if the following conditions hold:
\begin{equation*}
    \begin{aligned}
    \nabla h\left(z^*\right)\mathbb{R}^n&=\mathbb{R}^l,\\
    \nabla G\left(z^*\right)\ker\!\left(\nabla h\left(z^*\right)\right)
    +\operatorname{lin}\left(\mathcal{T}_{\coneK}\left(y^*\right)\right)&=\mathbb{R}^m,
    \end{aligned}
\end{equation*}
where $\operatorname{lin}(\mathcal{T}_{\coneK}(y^*))$ denotes the largest linear subspace contained in the tangent cone $\mathcal{T}_{\coneK}(y^*)$.
\end{definition}

Define the critical directions of Problem~\ref{prob:conic} at $z^*$ by
\begin{equation*}
    \mathcal{D}^*:=\left\{d\in\mathbb{R}^n:
    \nabla G(z^*)d\in\mathcal{C}_{\coneK}(y^*),\quad
    \nabla h(z^*)d=0\right\}.
\end{equation*}

For $y\in\coneK$ and $v\in\mathcal T_{\coneK}(y)$, define the
second-order tangent set by
\[
\mathcal T_{\coneK}^2(y,v)
:=
\left\{w:
\operatorname{dist}\!\left(y+tv+\tfrac12t^2w,\coneK\right)
=o(t^2)\ \text{as }t\downarrow0
\right\}.
\]
For a set $D$, let $\sigma(a,D):=\sup_{w\in D}\langle a,w\rangle$
denote its support function.

\begin{definition}[Second-order sufficient condition (SOSC)]
Problem~\ref{prob:conic} satisfies the SOSC at the optimal primal-dual solution
$(z^*,\mu^*,\nu^*)$ if, for every $d\in\mathcal D^*\setminus\{0\}$,
\[
d^T\nabla_{zz}^2\mathcal L(z^*,\mu^*,\nu^*)d
-
\sigma\!\left(
-\mu^*,
\mathcal T_{\coneK}^2
\bigl(y^*,\nabla G(z^*)d\bigr)
\right)>0.
\]
\end{definition}

The second-order sufficient condition requires the Lagrangian curvature,
together with the curvature of the cone, to be strictly positive along every
nonzero critical direction. Intuitively, this ensures sufficient curvature along the feasible
first-order directions relevant to optimality, making the stationary
solution locally isolated.

\begin{definition}[$\mathcal{C}^2$-cone reducibility]
\label{def:c2_cone_reducibility}
\citep[Definition~3.135]{bonnans2000perturbation}
A closed convex set $\coneK\subseteq\mathbb{R}^m$ is said
to be $\mathcal{C}^2$-cone reducible
at $y^*\in\coneK$ if there exist an open neighborhood $U$
of $y^*$, a pointed closed convex cone
$\mathcal{Q}\subseteq\mathbb{R}^q$, and a twice continuously
differentiable mapping $\Xi:U\to\mathbb{R}^q$,
for some integer $q\geq 0$, satisfying:
\begin{enumerate}
    \item[(i)] $\Xi(y^*)=0$;
    \item[(ii)] $D\Xi(y^*):\mathbb{R}^m\to\mathbb{R}^q$
    is onto;
    \item[(iii)] $\coneK\cap U
    =\{y\in U:\Xi(y)\in\mathcal{Q}\}$.
\end{enumerate}
\end{definition}

Intuitively,  $\mathcal{C}^2$-cone reducibility means that local
feasibility can be described by a twice continuously differentiable
mapping into a simpler pointed closed convex cone. The class of $\mathcal{C}^2$-cone reducible sets is rich, including all the polyhedral convex sets,
second-order cones, positive semidefinite cones, and epigraph
cones of Ky Fan matrix $k$-norms. This property is also preserved under finite Cartesian products \citep{bonnans2000perturbation}.

\begin{proposition}\label{lemma:active_face}\citep{ding2023strict}
Under strict complementarity and non-degeneracy, the active faces associated with the complementary pair $(y^*,\mu^*)$ satisfy
\begin{equation}
\begin{aligned}
    \mathcal{F}_{\coneK}(y^*) &= \mu^{*\perp} \cap \coneK \\
    \mathcal{F}_{\coneK^*}(\mu^*) &= - \mathcal{N}_{\coneK}(y^*) = y^{*\perp} \cap \coneK^{*}.
\end{aligned}
\end{equation}
\end{proposition}

Lemma~\ref{lemma:active_face} shows that the active faces follow from the KKT conditions under strict complementarity. In particular, the primal active face is obtained by combining primal cone feasibility~\Eqref{eqn:kkt_primal_cone} with complementarity~\Eqref{eqn:kkt_complementarity}.

\section{Proof of Theorem~\ref{theorem:conic_program}}
\label{app:proof_conic_program}

\begin{proof}

The KKT conditions of the parametric Problem~\ref{prob:conic} are:
\begin{equation}
\begin{aligned}
    \nabla_z f_{\theta}(z^*(\theta)) - \nabla_z h_{\theta}(z^*(\theta))^T\nu^*(\theta) - \nabla_z G_{\theta}(z^*(\theta))^T\mu^*(\theta) &= 0 \\
    h_{\theta}(z^*(\theta)) &= 0 \\
    y^*(\theta) &\in \coneK \\
    \mu^*(\theta) &\in \coneK^* \\
    \langle \mu^*(\theta), y^*(\theta)\rangle &= 0.
\end{aligned} \label{eqn:conic_kkt}
\end{equation}
Combining complementarity with primal and dual feasibility shows that $y^*$ and $\mu^*$ lie in their respective active faces: $y^* \in \coneK \cap \mu^{*\perp} = \mathcal{F}_{\coneK}(y^*)$ and $\mu^* \in \coneK^* \cap y^{*\perp} = \mathcal{F}_{\coneK^*}(\mu^*) = -\mathcal{N}_\coneK(y^*)$ by Lemma~\ref{lemma:active_face}.

Now we write the differential of the original Problem~\ref{prob:conic} evaluated at $(z^*,\mu^*,\nu^*)$ with fixed $\theta = \theta_0$. To simplify notation, once $\theta$ is fixed at $\theta_0$, we omit explicit subscripts and function arguments (\eg, $_\theta$ and $(\cdot)$); the same convention is used in the subsequent differential derivations. The differential of the KKT system is:

\begin{equation}
\begin{aligned}
    \begin{aligned}[b]
        \nabla_{zz}^2\mathcal{L} \dev z^*
        &+ \nabla_{z\theta}^2 f \dev \theta
        - \sum_{k=1}^{l}\nu_k^*\nabla_{z\theta}^2 h_k \dev \theta \\
        &- \sum_{i=1}^{m}\mu_i^*\nabla_{z\theta}^2 G_i \dev \theta
        - \nabla_z h^T \dev \nu^* - \nabla_z G^T \dev \mu^*
    \end{aligned} &= 0 \\
    \nabla_z h \dev z^* + \nabla_\theta h \dev \theta &= 0 \\
    \dev y^* &\in \mathcal{C}_{\coneK}(y^*)  \\
    \dev \mu^* &\in \mathcal{C}_{\coneK^*}(\mu^*),
\end{aligned} \label{eqn:conic_diff}
\end{equation}

where the last two conditions involve critical cones at the optimal solution, $\mathcal{C}_{\coneK}(y^*)$ and $\mathcal{C}_{\coneK^*}(\mu^*)$, as defined in Section~\ref{sec:conic_geometry}: $\mathcal{C}_{\coneK}(y^*) = \mathcal{T}_{\coneK}(y^*) \cap \mu^{*\perp}$, $\mathcal{C}_{\coneK^*}(\mu^*) = \mathcal{T}_{\coneK^*}(\mu^*) \cap y^{*\perp}$.

The differential of the complementarity condition can be omitted because the critical-cone conditions already imply $\langle \dev \mu^*,y^*\rangle=\langle \mu^*,\dev y^*\rangle=0$.

We next show that there exist matrices $H^*\in\mathbb{R}^{m\times m}$ and $C^*\in\mathbb{R}^{r_C\times m}$ such that $\dev\mu^*$ admits the representation
\begin{equation}
    \dev \mu^* = H^* \dev y^* + C^{*T}\dev \hat{\mu}. \label{eqn:dmu}
\end{equation}

By $C^2$-cone reducibility, there exist an open neighborhood $\mathcal Y$
of $y^*$, a pointed closed convex cone $D\subseteq\mathbb R^r$, and
a $C^2$ mapping $\Xi:\mathcal Y\to\mathbb R^r$ such that
\[
\Xi(y^*)=0,\qquad \Xi'(y^*)\text{ is onto},\qquad
\coneK\cap\mathcal Y=\{y\in\mathcal Y:\Xi(y)\in D\}.
\]
Here $\Xi'$ denotes the Jacobian. Since $\Xi'(y^*)$ has full row rank
and $\Xi'$ is continuous, we may choose $\mathcal Y$ small enough that
$\Xi'(y)$ has full row rank for every $y\in\mathcal Y$.
The normal-cone reduction formula
\citep[Lemma~2.1]{yin2019perturbation} gives
\[
\mathcal N_{\coneK}(y)=\Xi'(y)^T\mathcal N_D(\Xi(y))
\qquad(y\in\coneK\cap\mathcal Y).
\]
For each $\theta$ near $\theta_0$, the KKT conditions give
$-\mu^*(\theta)\in\mathcal N_{\coneK}(y^*(\theta))$.
Applying the reduction formula at $y^*(\theta)$ therefore gives a unique
$u(\theta)\in\mathcal N_D(\Xi(y^*(\theta)))$ such that
\[
-\mu^*(\theta)=\Xi'(y^*(\theta))^Tu(\theta).
\]
Since $\Xi'$ is $C^1$ and has full row rank locally, $u(\theta)$ is
differentiable along the differentiable primal--dual solution branch.
Set $u^*:=u(\theta_0)\in D^\circ$, where $D^\circ$ is the negative polar.
Since $\Xi'(y^*)^T$ is injective, strict complementarity implies
$u^*\in\operatorname{ri}D^\circ=\operatorname{int}D^\circ$.
Consequently, $D\cap(u^*)^\perp=\{0\}$.

To identify the critical cone, first note that surjectivity of
$\Xi'(y^*)$ ensures metric regularity, and hence metric subregularity,
of $y\mapsto\Xi(y)-D$ at $(y^*,0)$.
Applying \citet[Corollary~2.1]{liu2019computation} to the local
representation $\coneK=\Xi^{-1}(D)$ therefore gives
\[
\begin{aligned}
\mathcal T_{\coneK}(y^*)
&=\{v:\Xi'(y^*)v\in\mathcal T_D(\Xi(y^*))\}\\
&=\{v:\Xi'(y^*)v\in D\},
\end{aligned}
\]
where the last equality uses $\Xi(y^*)=0$ and $\mathcal T_D(0)=D$.
By the definition of the critical cone and
$-\mu^*=\Xi'(y^*)^Tu^*$, it follows that
\[
\begin{aligned}
\mathcal{C}_{\coneK}(y^*)
&=\mathcal T_{\coneK}(y^*)\cap(\mu^*)^\perp\\
&=\{v:\Xi'(y^*)v\in D,\ \langle u^*,\Xi'(y^*)v\rangle=0\}\\
&=\{v:\Xi'(y^*)v\in D\cap(u^*)^\perp\}
=\ker\Xi'(y^*).
\end{aligned}
\]
We therefore choose
\[
\begin{aligned}
C(y)&:=\Xi'(y),\qquad C^*:=C(y^*),\\
H^*&:=-\nabla_y^2\langle u^*,\Xi\rangle(y^*)
=-\sum_{i=1}^r u_i^*\nabla_y^2\Xi_i(y^*).
\end{aligned}
\]
Then $\ker C^*=\mathcal{C}_{\coneK}(y^*)$ and $H^*$ is symmetric.
Taking $\hat\mu(\theta_0)=-u^*$ gives
$\mu^*=C^{*T}\hat\mu(\theta_0)$.

Differentiating the multiplier representation at $\theta_0$ gives
\[
\begin{aligned}
\dev\mu^*
&=-\left(\sum_{i=1}^r u_i^*\nabla_y^2\Xi_i(y^*)\right)\dev y^*
-\Xi'(y^*)^T\dev u\\
&=H^*\dev y^*-C^{*T}\dev u.
\end{aligned}
\]
The KKT comparison below identifies $\dev\hat\mu=-\dev u$ for the
reduced-problem multiplier, yielding~\Eqref{eqn:dmu}.

Moreover, differentiating $C(y)=\Xi'(y)$ gives,
for every $v\in\mathbb R^m$,
\[
DC(y^*)[v]^T\hat\mu(\theta_0)
=-\sum_{i=1}^r u_i^*\nabla_y^2\Xi_i(y^*)v
=H^*v.
\]
then set $C_\theta^*:=C(G_\theta(z^*))$, with the reference $z^*$
held fixed. The chain rule yields
\begin{equation}
\dev C_\theta^{*T}\hat\mu(\theta_0)
=H^*\nabla_\theta G_{\theta_0}(z^*)\dev\theta.\label{eqn:dC_H}
\end{equation}
This construction yields a compatible pair $C_{\theta}^*$ and $H^*$, with \Eqref{eqn:dmu} and \Eqref{eqn:dC_H} are the two identities needed for the KKT comparison below.

Substituting \Eqref{eqn:dmu} into the differential KKT system~\Eqref{eqn:conic_diff} yields
\begin{equation}
\begin{aligned}
    \begin{aligned}[b]
        \nabla_{zz}^2\mathcal{L} \dev z^*
        &+ \nabla_{z\theta}^2 f \dev \theta
        - \sum_{k=1}^{l}\nu_k^*\nabla_{z\theta}^2 h_k \dev \theta \\
        &- \sum_{i=1}^{m}\mu_i^*\nabla_{z\theta}^2 G_i\dev \theta
        - \nabla_z h^T \dev \nu^* \\
        &- \nabla_z G^T (H^* \dev y^* + C^{*T}\dev \hat{\mu})
    \end{aligned} &= 0 \\
    \nabla_z h \dev z^* + \nabla_\theta h \dev \theta &= 0 \\
    C^* (\nabla_z G \dev z^* + \nabla_\theta G \dev \theta) &= 0
\end{aligned} \label{eqn:reduced_conic_diff}
\end{equation}
The third equation is obtained by noting that $\dev y^* = \nabla_z G\dev z^* + \nabla_\theta G \dev \theta \in \mathcal{C}_\coneK(y^*) = \ker{C^*}$.

Finally, we want to show the first-order equivalence between the reduced problem and the original problem, and $\hat{\mu}$ is indeed the dual solution of reduced Problem~\ref{prob:reduced_conic}.

For reduced Problem~\ref{prob:reduced_conic}, denote $\hat{\mu}$ and $\hat{\nu}$ as the dual variables for the reduced conic constraint $C_\theta^*\bigl((G_{\theta}(z^*)-y^*) + \nabla_z G_{\theta}(z^*)(z-z^*)\bigr) = 0$ and the affine constraint $h_{\theta}(z) = 0$, respectively. The KKT conditions are:
\begin{equation}
\begin{aligned}
    \begin{aligned}[b]
        &\Bigl(\nabla^{2}_{zz} \mathcal{L}_{\theta}(z^*)
        - \nabla_z G_{\theta}(z^*)^{T}H^* \nabla_z G_{\theta}(z^*)\Bigr)
        \bigl(\hat{z}(\theta)-z^*\bigr) \\
        &\quad + \nabla_{z}f_{\theta}(z^*)
        - \nabla_z h_{\theta}(\hat{z}(\theta))^T\hat{\nu}(\theta) \\
        &\quad - \nabla_z G_{\theta}(z^*)^T C_\theta^{*T}\hat{\mu}(\theta)
    \end{aligned} &= 0 \\
    h_\theta(z^*)+\nabla_z h_\theta(z^*)
\bigl(\hat z(\theta)-z^*\bigr) &= 0 \\
    C_\theta^*\Bigl(G_{\theta}(z^*)-y^*
    + \nabla_z G_{\theta}(z^*)\bigl(\hat{z}(\theta)-z^*\bigr)\Bigr) &= 0
\end{aligned} \label{eqn:reduced_conic_kkt}
\end{equation}

Under the standing assumptions, the system has the unique solution $z^*=\hat z(\theta_0)$, $\mu^*=C^{*T}\hat\mu(\theta_0)$, and $\nu^*=\hat\nu(\theta_0)$ at $\theta_0$.

The differential of the KKT system \Eqref{eqn:reduced_conic_kkt} at $\theta_0$ is:
\begin{equation}
\begin{aligned}
    \begin{aligned}[b]
        \left(\nabla^{2}_{zz} \mathcal{L} - \nabla_z G^{T}H^* \nabla_z G\right) \dev \hat{z}
        &+ \nabla_{z\theta}^2 f \dev \theta
        - \sum_{k=1}^{l}\hat{\nu}_k\nabla_{z\theta}^2 h_k \dev \theta \\
        &- \nabla_z h^T \dev \hat{\nu}
        - \sum_{i=1}^{m}(C^{*T}\hat{\mu})_i\nabla_{z\theta}^2 G_i\dev \theta \\
        &- \nabla_zG^TH^*\nabla_\theta G\dev\theta \\
        &- \nabla_z G^T C^{*T}\dev \hat{\mu}
    \end{aligned} &= 0 \\
    \nabla_z h \dev \hat{z} + \nabla_\theta h \dev \theta &= 0 \\
    C^* \nabla_z G \dev \hat{z} + C^* \nabla_\theta G \dev \theta &= 0
\end{aligned} \label{eqn:reduced_conic_diff_kkt}
\end{equation}
Here the term, $\nabla_zG^TH^*\nabla_\theta G\dev\theta $, follows from
\Eqref{eqn:dC_H};
Note that \Eqref{eqn:reduced_conic_diff} is equivalent to \Eqref{eqn:reduced_conic_diff_kkt}, which also shows $\hat{\mu}$ is indeed the dual solution.
Rewriting \Eqref{eqn:reduced_conic_diff_kkt} in block-matrix form yields
\Eqref{eqn:reduced_conic_diff_kkt_matrix}.

\end{proof}

\subsection{A Concrete Example for Scalar Cone Constraints}

We illustrate the decomposition of $\dev\mu^*$ in~\Eqref{eqn:dmu} for a cone
locally defined at $y^*$ by a twice continuously differentiable scalar-valued
function $\phi:\mathbb{R}^m\to\mathbb{R}$ with
$\phi(y^*)=0$ and $\nabla\phi(y^*)\neq0$:
$\coneK=\{y:\phi(y)\leq0\}$.

By strict complementarity, there exists a scalar $\alpha>0$ such that
$$\mu^* = -\alpha \nabla_y \phi(y^*).$$

Differentiating this relation with respect to the parameters, we obtain
$$\dev \mu^* = -\alpha \nabla_{yy}^2 \phi(y^*) \dev y^* - \dev \alpha \nabla_y \phi(y^*).$$

Comparing with \Eqref{eqn:dmu}, we can identify the components:
\begin{itemize}
    \item $H^*=-\alpha \nabla_{yy}^2 \phi(y^*)$ captures the curvature effect (how the normal direction rotates as $y^*$ moves along the cone boundary),
    \item $C^*=\nabla_y \phi(y^*)^{T}$ provides the basis for the normal space,
    \item $\dev\hat{\mu}=-\dev \alpha$ is the coefficient for the normal component.
\end{itemize}

\section{Complementary Details for Common Reductions (Section~\ref{sec:instances})}
\label{app:instance_details}

\subsection{Details on Second-Order Cone Programs (Section~\ref{sec:socp})}
\label{app:socp}
\begin{proof}[\textbf{Proof of Lemma~\ref{lemma:socp_H_C}}]
For $j\in\Jna$, the function $\phi(x,t)=\|x\|_2-t$ is smooth
near $y_j^*$, with $\phi(y_j^*)=0$ and $\nabla\phi(y_j^*)\neq0$.
Thus $\Xi_j=\phi$ is a local reduction to $\mathbb R_-$.
Strict complementarity gives
\[
\mu_j^*=-\alpha_j\nabla\phi(y_j^*),\qquad \alpha_j>0.
\]
Since the last component of $\nabla\phi$ is $-1$ and
$\mu_j^*=(u_j^*,s_j^*)$, we have $\alpha_j=s_j^*$.
The construction in Appendix~\ref{app:proof_conic_program} therefore yields
\[
H_j^*=-s_j^*\nabla^2\phi(y_j^*),\qquad C_{j,\theta}^* = \nabla\phi(G_j(z^*))^T \qquad
C_j^*:=C_{j,\theta_0}^*=\nabla\phi(y_j^*)^T.
\]
Thus, $\ker C_j^*=\{\mathrm dy:\nabla\phi(y_j^*)^T\mathrm dy=0\}$, which is the critical cone in~\Eqref{eq:socp_critical_cone}.

For $j\in\Jap$, use the identity reduction $\Xi_j(y_j)=y_j$
to $\mathcal K_{\mathrm{SOC}_j}$. Its Jacobian is $I$ and its
second derivative vanishes, so the same construction gives
\[
H_j^*=0,\qquad C_{j,\theta}^*=C_j^*=I.
\]
In particular, $\ker C_j^*=\{0\}$ agrees with the critical cone
in~\Eqref{eq:socp_critical_cone}.
\end{proof}

\subsection{Details on Semidefinite Programs (Section~\ref{sec:sdp})}
\label{app:sdp}
We provide the details underlying
Lemma~\ref{lemma:sdp_H_C}.
Section~\ref{sec:conic_geometry} and the proof of
Theorem~\ref{theorem:conic_program} give
\begin{align*}
    Z^* \in \mathcal{F}_{\coneK}(Z^*), &\quad\dev Z^*\in \mathcal{C}_{\coneK}(Z^*) \\
    S^* \in \mathcal{F}_{\coneK^*}(S^*)=-\mathcal{N}_{\coneK}(Z^*), &\quad\dev S^* \in \mathcal{C}_{\coneK^*}(S^*)
\end{align*}

Define the basis transform operator $T^*: \mathbb S^n \rightarrow \mathbb S^n$ by
$$T^*(X) = U^{*T} X U^* = \begin{bmatrix} T^*(X)_{00} & T^*(X)_{01}\\
    T^*(X)_{01}^{T} & T^*(X)_{11}
\end{bmatrix},$$
where $T^*(X)_{00} \in \mathbb{R}^{|J| \times |J|}$, $T^*(X)_{01} \in \mathbb{R}^{|J| \times (n-|J|)}$, and $T^*(X)_{11} \in \mathbb{R}^{(n-|J|) \times (n-|J|)}$. Note that $U^*$ is the eigenbasis of $Z^*$, which is fixed and does not depend on the perturbation, thus we can treat $U^*$ as a constant matrix and ignore its differential when deriving the reduced problem.

Now we introduce the cone geometries at the optimal solution.
\begin{proposition}\label{lemma:S_perp}
\begin{equation}
S^{*\perp} = \{Z \in \mathbb{S}_{n}: \ip{T^*(Z)_{00},T^*(S^*)_{00}} = 0 \}
\end{equation}
\end{proposition}

\begin{proof}[\textbf{Proof of Lemma~\ref{lemma:S_perp}.}]
\begin{align*}
S^{*\perp} &= \{Z \in \mathbb{S}_{n}: \ip{Z,S^*} = 0\} \\
& = \{Z \in \mathbb{S}_{n}: \ip{T^*(Z),T^*(S^*)}
 = 0\} \\
 &= \{Z \in \mathbb{S}_{n}: \ip{\begin{bmatrix}
    T^*(Z)_{00} & T^*(Z)_{01} \\
    T^*(Z)_{01}^{T} & T^*(Z)_{11}
\end{bmatrix},\begin{bmatrix}
    T^*(S^*)_{00} & 0 \\
    0 & 0
\end{bmatrix}}= 0\} \\
&= \{Z \in \mathbb{S}_{n}: \ip{T^*(Z)_{00},T^*(S^*)_{00}} = 0\}
\end{align*}
\end{proof}

\begin{proposition}\label{lemma:active_face_sdp}
The active face of $Z^*$, $\mathcal{F}_{\coneK}(Z^*) = \coneK \cap S^{*\perp}$, is equivalent to
\begin{equation}
\mathcal{F}_{\coneK}(Z^*) = \{Z \in \mathbb{S}^{+}_{n}: Z\succeq 0, \ip{S^*,Z} = 0\} = \{Z: T^*(Z) = \begin{bmatrix}
    0 & 0 \\
    0 & T^*(Z)_{11}
\end{bmatrix}, T^*(Z)_{11} \succeq 0\}
\end{equation}
\end{proposition}

\begin{proof}[\textbf{Proof of Lemma~\ref{lemma:active_face_sdp}.}]
Under strict complementarity, $T^*(S^*)_{00}\succ0$.
Since $T^*(Z)_{00}\succeq0$, the identity
$\langle T^*(Z)_{00},T^*(S^*)_{00}\rangle=0$
implies that $T^*(Z)_{00}=0$.
The positive semidefiniteness of $T^*(Z)$ then yields
$T^*(Z)_{01}=0$.
\end{proof}

\begin{proposition}
    \label{prop:tangent_cone_sdp}
    The tangent cone of $\coneK$ at $Z^*$ is equivalent to
\begin{equation}
\mathcal{T}_\coneK(Z^*) = \{\dev Z\in\mathbb{S}_{n}: T^*(\dev Z)_{00} \succeq 0\}
\end{equation}
\end{proposition}

\begin{proof}[\textbf{Proof of Lemma~\ref{prop:tangent_cone_sdp}.}]
We first prove necessity. If
$\dev Z\in\mathcal{T}_{\coneK}(Z^*)$, then there exist
$\alpha_k\geq0$ and $Z_k\succeq0$ such that
$\alpha_k(Z_k-Z^*)\to\dev Z$. Since $T^*(Z^*)_{00}=0$, the $00$
principal block of $\alpha_kT^*(Z_k-Z^*)$ is
$\alpha_kT^*(Z_k)_{00}\succeq0$. Taking the limit gives
$T^*(\dev Z)_{00}\succeq0$.

Conversely, suppose that $T^*(\dev Z)_{00}\succeq0$. For any
$\epsilon>0$, define $\dev Z_\epsilon$ by
\begin{equation}
T^*(\dev Z_\epsilon)
=T^*(\dev Z)+
\begin{bmatrix}
\epsilon I&0\\
0&0
\end{bmatrix}.
\end{equation}
For sufficiently large $\alpha>0$,
$T^*(Z^*)_{11}+\alpha^{-1}T^*(\dev Z)_{11}\succ0$. Moreover, since
$T^*(\dev Z)_{00}+\epsilon I\succ0$, the corresponding Schur complement
satisfies
\begin{equation}
\alpha^{-1}\bigl(T^*(\dev Z)_{00}+\epsilon I\bigr)
-\alpha^{-2}T^*(\dev Z)_{01}
\bigl(T^*(Z^*)_{11}+\alpha^{-1}T^*(\dev Z)_{11}\bigr)^{-1}
T^*(\dev Z)_{01}^{T}\succeq0
\end{equation}
for all sufficiently large $\alpha>0$. Therefore,
\begin{equation}
T^*(Z^*+\alpha^{-1}\dev Z_\epsilon)
=
\begin{bmatrix}
\alpha^{-1}\bigl(T^*(\dev Z)_{00}+\epsilon I\bigr)&\alpha^{-1}T^*(\dev Z)_{01}\\
\alpha^{-1}T^*(\dev Z)_{01}^{T}&T^*(Z^*)_{11}+\alpha^{-1}T^*(\dev Z)_{11}
\end{bmatrix}
\succeq0.
\end{equation}
Setting $Z_\epsilon:=Z^*+\alpha^{-1}\dev Z_\epsilon\succeq0$ gives
$\dev Z_\epsilon=\alpha(Z_\epsilon-Z^*)$, and hence
$\dev Z_\epsilon\in\mathcal{T}_{\coneK}(Z^*)$. Since
$\dev Z_\epsilon\to\dev Z$ as $\epsilon\downarrow0$ and the tangent cone
is closed, $\dev Z\in\mathcal{T}_{\coneK}(Z^*)$.

Thus,
\begin{equation}
\mathcal{T}_{\coneK}(Z^*)
=\{\dev Z\in\mathbb{S}_n:T^*(\dev Z)_{00}\succeq0\}.
\end{equation}
\end{proof}

The conditions for the dual variable $S^*$ are similar. We have
\begin{proposition}
\begin{equation}
Z^{*\perp} = \{S\in \mathbb{S}_{n}: \ip{T^*(S)_{11}, T^*(Z^*)_{11}} = 0 \}
\end{equation}
\end{proposition}

\begin{proposition}\label{lemma:active_face_dual_sdp}
    The active face of $S^*$, $\mathcal{F}_{\coneK^*}(S^*) = \coneK^* \cap Z^{*\perp}= -\mathcal{N}_{\coneK}(Z^*) $, is equivalent to
\begin{equation}
\begin{aligned}
\mathcal{F}_{\coneK^*}(S^*) &= -\mathcal{N}_{\coneK}(Z^*)
 = \{S\in \mathbb{S}_{n}: S\succeq 0, SZ^* = 0\} \\
&= \left\{S: T^*(S) = \begin{bmatrix}
    T^*(S)_{00} & 0 \\
    0 & 0
\end{bmatrix},\quad T^*(S)_{00} \succeq 0 \right\}.
\end{aligned}
\end{equation}
\end{proposition}

\begin{proposition}
     The tangent cone of $\coneK^*$ at $S^*$ is
\begin{equation}
\mathcal{T}_{\coneK^*}(S^*) = \{\dev S\in \mathbb{S}_{n}: T^*(\dev S)_{11} \succeq 0\}
\end{equation}
\end{proposition}

\begin{proposition}
    \label{lemma:critical_cone_dual_sdp}
    The critical cone, $\mathcal{C}_{\coneK^*}(S^*) = \mathcal{T}_{\coneK^*}(S^*) \cap Z^{*\perp}$, is equivalent to
\begin{equation}
 \mathcal{C}_{\coneK^*}(S^*) = \{\dev S\in \mathbb{S}_{n}: T^*(\dev S)_{11} = 0 \}
\end{equation}
\end{proposition}

With these cone-geometric characterizations in place, we are ready to prove Lemma~\ref{lemma:sdp_H_C}.

\begin{proof}[\textbf{Proof of Lemma~\ref{lemma:sdp_H_C}}.]
In this problem, we have $G_{\theta}(Z) = G_{\theta_0}(Z) = Z$. Thus we omit the subscript $\theta$ in the following derivations.

Differentiating the complementarity condition $S^*Z^*=0$ gives
\begin{equation}
    \dev S^* Z^* + S^* \dev Z^* = 0
\end{equation}
This is equivalent to
\begin{equation}\label{eqn:dev_S_Z}
     T^*(\dev S^*) T^*(Z^*)  + T^*(S^*) T^*(\dev Z^*) = 0
\end{equation}
Note that this step only multiplies by the fixed matrix $U^*$, so $\dev U^*$ does not appear in the following derivations.

Substituting the cone structures $\dev S^*\in\mathcal{C}_{\coneK^*}(S^*)$, $Z^*\in\mathcal{F}_\coneK(Z^*)$, $S^*\in\mathcal{F}_{\coneK^*}(S^*)$, and $\dev Z^*\in\mathcal{C}_{\coneK}(Z^*)$, given by Lemmas~\ref{lemma:critical_cone_dual_sdp}, \ref{lemma:active_face_sdp}, and~\ref{lemma:active_face_dual_sdp} and by~\Eqref{eqn:sdp_critical_cone}, into~\Eqref{eqn:dev_S_Z} yields
\begin{align}
    \begin{bmatrix}
        T^*(\dev S^*)_{00} & T^*(\dev S^*)_{01} \\
        T^*(\dev S^*)_{01}^{T} & 0
    \end{bmatrix} \begin{bmatrix}
        0 & 0 \\
        0 & T^*(Z^*)_{11}
    \end{bmatrix} + \begin{bmatrix}
        T^*(S^*)_{00} & 0 \\
        0 & 0
\end{bmatrix} \begin{bmatrix}
        0 & T^*(\dev Z^*)_{01}\\
        T^*(\dev Z^*)_{01}^{T} & T^*(\dev Z^*)_{11}
\end{bmatrix} & = 0 \\
\Leftrightarrow \begin{bmatrix}
        0 & T^*(\dev S^*)_{01} T^*(Z^*)_{11} + T^*(S^*)_{00} T^*(\dev Z^*)_{01} \\
        0 & 0
\end{bmatrix} & = 0
\end{align}
Thus, $T^*(\dev S^*)_{01} = - T^*(S^*)_{00} T^*(\dev Z^*)_{01} T^*(Z^*)_{11}^{-1}$, and $\dev S^*$ can be represented as
\begin{equation}
    \dev S^* = U^* \begin{bmatrix}
        0 & T^*(\dev S^*)_{01} \\
        T^*(\dev S^*)_{01}^{T} & 0
    \end{bmatrix} U^{*T} + U^* \begin{bmatrix}
        T^*(\dev S^*)_{00} & 0 \\
        0 & 0
    \end{bmatrix} U^{*T}
\end{equation}
After some rewriting, we have
\begin{equation}
    \dev S^* = -\left(Z^{*\dagger} \dev Z^* S^* +  S^* \dev Z^* Z^{*\dagger}\right) + \sum_{j,j' \in J} T^*(\dev S^*)_{00,jj'} u^*_j u^{*T}_{j'} \label{eqn:dS}
\end{equation}
The first term describes the change in $S^*$ along the critical cone $\mathcal{C}_{\coneK}(Z^*)$, whereas the second lies in $\mathcal{C}_{\coneK}(Z^*)^{\perp}$. Vectorizing~\Eqref{eqn:dS} gives

\begin{equation}
    \begin{aligned}
    \vecm{\dev S^*} &= -(Z^{*\dagger} \otimes S^* + S^* \otimes Z^{*\dagger}) \vecm{\dev Z^*} + (U^*_0 \otimes U^*_0) \vecm{T^*(\dev S^*)_{00}} \\
     &= -(Z^{*\dagger} \otimes S^* + S^* \otimes Z^{*\dagger}) \vecm{\dev Z^*} + \sum_{j,j' \in J} T^*(\dev S^*)_{00,jj'} \vecm{u^*_j u^{*T}_{j'}}
    \end{aligned}
\end{equation}

From the first term, we have $H^* = -(Z^{*\dagger} \otimes S^* + S^* \otimes Z^{*\dagger})$.

Since $T^*(\dev S^*)_{00}$ is symmetric, grouping the symmetric terms in the
second term of \Eqref{eqn:dS} gives $C^*$ with rows
$\vecm{B_{jj'}}^T$ for $j,j'\in J$ with $j'\leq j$, where
\begin{equation}
B_{jj}=u_j^*u_j^{*T},
\qquad
B_{jj'}=u_j^*u_{j'}^{*T}+u_{j'}^*u_j^{*T},
\quad j'<j.
\end{equation}
The corresponding coefficients form $\dev\hat{\mu}$.
\end{proof}

\paragraph{Example: multiple zero eigenvalues and the critical cone.}
Consider $Z^*=\operatorname{diag}(0,0,1)\in\mathbb S_+^3$ and the
perturbations $\delta Z_1=\epsilon e_1e_1^T$ and
$\delta Z_2=\epsilon e_2e_2^T$, where $\epsilon>0$ and $e_1,e_2$ are the
standard basis vectors associated with the two zero eigenvalues. Although
$Z^*+\delta Z_1$ and $Z^*+\delta Z_2$ have the same eigenvalues
$(0,\epsilon,1)$, their normal cones are
\begin{align*}
\mathcal N_{\mathbb S_+^3}(Z^*+\delta Z_1)
  &=\{-\alpha e_2e_2^T:\alpha\geq0\},\\
\mathcal N_{\mathbb S_+^3}(Z^*+\delta Z_2)
  &=\{-\alpha e_1e_1^T:\alpha\geq0\}.
\end{align*}
Thus, arbitrarily small perturbations within the zero-eigenspace can select
different normal directions, precluding a single smooth local representation
of the normal-cone mapping along these directions.

Under strict complementarity, however, the critical-cone condition
$U_0^{*T}\delta ZU_0^*=0$, with $U_0^*=[e_1,e_2]$, restricts admissible
perturbations to the form
\[
\delta Z=\begin{bmatrix}0&0&a\\0&0&b\\a&b&c\end{bmatrix}.
\]
It therefore excludes all perturbations within the zero-eigenspace, including
$\delta Z_1$ and $\delta Z_2$, while retaining cross-eigenspace and
positive-eigenspace directions. In these admissible directions, the curvature
term used by the reduced problem is well defined.

\section{Additional Experimental Details}
Experiments in this paper were run on an M2 MacBook Air with 8 CPU cores and 16GB RAM.

We corrected a variable-name error in the latest publicly available
FFOLayer implementation at the time of our experiments (commit \texttt{28905f3e1750fca5b8918954d5d2ea5bed0cbacc}). Before this one-line fix, its gradients differed substantially from those produced by both dOPT and CVXPYLayers; after the fix, all three methods are numerically consistent. All reported \emph{FFOLayer (Modified)} results use the corrected code, which is included in our repository.

\label{app:experiment_details}
\subsection{Details for Synthetic Experiments (Section~\ref{sec:synthetic_experiments})}
For each problem class and dimension, we evaluate ten randomly generated
instances and report the mean of each metric.

\paragraph{Random SOCPs.} We generate problems of the form
\begin{equation}
    \min_z\ q^\top z
    \quad\text{s.t.}\quad
    \lVert A_j z+b_j\rVert_2\leq c_j^\top z+d_j,\quad
    Fz=g,\quad Gz\leq h.
\end{equation}
For a decision dimension $n$, we use $\lfloor n/2\rfloor$ second-order
cones, each with vector dimension $\lfloor n/3\rfloor$, together with two
equalities and ten linear inequalities. We draw $q$, $A_j$, $b_j$, $c_j$,
$F$, and $G$ independently from a standard normal distribution. After drawing
$z_0\sim\mathcal N(0,I)$, we set
$d_j=\lVert A_j z_0+b_j\rVert_2-c_j^\top z_0$, $g=Fz_0$, and
$h=Gz_0+\mathbf 1$, ensuring that every generated instance is feasible.

\paragraph{Random SDPs.} We generate
\begin{equation}
    \min_{Z\succeq0}\ \langle C,Z\rangle
    \quad\text{s.t.}\quad A\operatorname{vec}(Z)=b,
\end{equation}
where $Z\in\mathbb S^n$ and
$m=\lfloor(n(n+1)/2-1)/2\rfloor$ equality constraints are used. We construct
$C=Q\Lambda Q^\top$, where $Q$ is obtained from the QR factorization of a
Gaussian matrix and the diagonal entries of $\Lambda$ are logarithmically
spaced from $1$ to $10$. The entries of $A\in\mathbb R^{m\times n^2}$ are
drawn independently from a standard normal distribution. Finally, we construct
a positive-definite matrix $Z_0=U\Lambda_0U^\top$ in the same way, with
eigenvalues logarithmically spaced from $1$ to $10^{1.5}$, and set
$b=A\operatorname{vec}(Z_0)$ to ensure feasibility.

\subsection{Details for vehicle platoon benchmark (Section~\ref{sec:switched_system})}
\label{app:platoon_details}

\paragraph{Benchmark. }
We adapt the networked cooperative vehicle-platoon benchmark of
\citet{makhlouf2014networked,chen2015benchmark}. A platoon with $N$
controlled vehicles has state
\begin{equation}
    x=(e_1,\dot e_1,a_1,\ldots,e_N,\dot e_N,a_N)\in\mathbb{R}^{3N},
\end{equation}
where $e_j$, $\dot e_j$, and $a_j$ are the spacing error, relative velocity,
and acceleration of vehicle $j$, respectively. We use the $N\in\{3,5,10\}$
instances, giving state dimensions $9$, $15$, and $30$. We index the two modes
consistently as $0$ (normal communication) and $1$ (total communication loss):
\begin{equation}
    \dot x=A^{\mathrm{ct}}_{\sigma(t)}x+B^{\mathrm{ct}}_L a_L(t),
    \qquad \sigma(t)\in\{0,1\},
\end{equation}
where $a_L$ is the acceleration command of the manually driven leader. The
leader is external to the $3N$-dimensional state, and its input acts through
$B^{\mathrm{ct}}_L$. Both $a_L(t)$ and $B^{\mathrm{ct}}_L$ are treated as known;
they are not inferred by the learning model.

For the three-vehicle instance, we use the two $9\times9$ continuous-time
matrices printed in the original benchmark. The five- and ten-vehicle
instances provide only the connected dynamics, so we extend the
communication-loss pattern of the three-vehicle instance.

We convert each continuous-time pair $(A_i^{\mathrm{ct}},B_L^{\mathrm{ct}})$
to the discrete-time model
\begin{equation}
    x_{k+1}=A_i x_k+B_i u_k
\end{equation}
by exact zero-order-hold discretization with sampling period
$\Delta t=0.1$.

The input matrices $B_i$ are treated as known physical quantities throughout;
the sampled input $u_k$ is also observed. Only the two mode matrices $A_0,A_1$
and the switching time are learned.

\paragraph{Trajectory data.}
We generate $16$ trajectories of length $T=30$ from the same ground-truth
switching sequence. Initial states are sampled independently as
$x_0\sim0.35\mathcal{N}(0,I)$. To model a leader acceleration command rather
than white-noise actuation, $u_k$ is piecewise constant: every five time steps
a new value is sampled uniformly from $[-2,1]$. The ground-truth system uses
mode $0$ for transitions $k<14$ and mode $1$ for $k\geq14$. We add independent
Gaussian observation noise to every state coordinate, with standard deviation
equal to $1\%$ of that coordinate's empirical standard deviation over the
clean training trajectories. Ground-truth mode labels are retained only for
evaluation and are not exposed to initialization or training.

\paragraph{Initialization.}
We first fit a single pooled dynamics matrix using every observed transition
\begin{equation}
    \bar A=\arg\min_A
    \sum_{j,k}\left\|x_{k+1}^{(j)}-Ax_k^{(j)}-\bar B u_k^{(j)}\right\|_2^2
    +\lambda_{\mathrm{ridge}}\|A\|_F^2,
    \qquad \lambda_{\mathrm{ridge}}=10^{-5},
\end{equation}
where $\bar B=(B_0+B_1)/2$. To avoid ill-conditioned SDP, we break the otherwise exact mode symmetry using
\begin{equation}
    A_0^{(0)}=\bar A-\Delta,
    \qquad A_1^{(0)}=\bar A+\Delta,
    \qquad
    \Delta=10^{-4}\|\bar A\|_F\frac{R}{\|R\|_F},
\end{equation}
with a dense Gaussian matrix $R$. The switch
time is initialized to $\tau^{(0)}=20.3$.

\paragraph{Model architecture and training.}
The learning architecture contains a single differentiable SDP layer, and
$A_0$ and $A_1$ are trainable parameters. For the switching time, we use a
soft-to-hard switching-time schedule. In the numerical formulation, we enforce
$P\succeq10^{-3}I$ and choose
$\alpha=0.015 d$. Specifically, we parameterize the
switching time by an unconstrained scalar $s$, mapped to the valid range as
\begin{equation}
    \tau=(T-1)\operatorname{sigmoid}(s).
\end{equation}
During the first $50$ epochs, we use the differentiable gate
\begin{equation}
    q_k=\operatorname{sigmoid}\!\left(\frac{k-\tau}{\gamma}\right),
    \qquad
    \widetilde A_k=(1-q_k)A_0+q_kA_1,
    \qquad
    \widetilde B_k=(1-q_k)B_0+q_kB_1.
\end{equation}
Here $q_k\in(0,1)$ is the soft probability of using mode $1$ at transition
$k$. Consequently, $\widetilde A_k$ and $\widetilde B_k$ are the effective
time-dependent dynamics used only to propagate the soft-switch model; they are
convex combinations of the two modes, not additional learned matrices. Starting
from the observed initial state, the soft rollout is defined recursively by
\begin{equation}
    \widehat x_0=x_0,
    \qquad
    \widehat x_{k+1}=\widetilde A_k\widehat x_k+\widetilde B_k u_k,
\end{equation}
and its fitting loss is the MSE between the complete predicted and observed
trajectories, $\mathcal{L}_{\mathrm{roll}}
=\operatorname{MSE}((\widehat x_k)_{k=0}^{T-1},(x_k)_{k=0}^{T-1})$.
The temperature $\gamma$ is geometrically annealed from $3.0$ to $0.25$.
This soft stage permits gradients to identify $\tau$ jointly with the two mode
matrices. We use Adam with learning rate $10^{-3}$ for $(A_0,A_1)$ and
$5\times10^{-2}$ for $s$.

\begin{table}[tbp]
    \caption{Performance on the adapted vehicle-platoon benchmark for
    $N\in\{3,5,10\}$ vehicles, with state dimension $d=3N$. Training
    entries are total times for all epochs.}
    \label{tab:platoon_stability}
    \centering
    \begingroup
\centering
\setlength{\tabcolsep}{4pt}
\renewcommand{\arraystretch}{0.96}
\resizebox{\textwidth}{!}{%
\begin{tabular*}{1.30\textwidth}{@{\extracolsep{\fill}}llccc@{}}
\toprule
Method & Metric & & & \\
\cmidrule(lr){1-5}
& Number of vehicles $N$ & 3 & 5 & 10 \\
& State dimension $d$ & 9 & 15 & 30 \\
\midrule
\multirow{4}{*}{dOPT}
& Stability margin $t^\star$ & $8.4\!\times\!10^{-3}$ & $6.3\!\times\!10^{-3}$ & $3.4\!\times\!10^{-3}$ \\
& Forward [ms] & $\mathbf{12.7}$ & $41.4$ & $468.4$ \\
& Backward [ms] & $\mathbf{1.4}$ & $\mathbf{3.6}$ & $\mathbf{26.1}$ \\
& Training [s] & $\mathbf{4.3}$ & $\mathbf{12.0}$ & $\mathbf{128.8}$ \\
\midrule
\multirow{4}{*}{FFOLayer (Mod)}
& Stability margin $t^\star$ & $8.5\!\times\!10^{-3}$ & $6.1\!\times\!10^{-3}$ & $3.4\!\times\!10^{-3}$ \\
& Forward [ms] & $14.3$ & $\mathbf{41.0}$ & $\mathbf{459.0}$ \\
& Backward [ms] & $16.2$ & $43.6$ & $472.7$ \\
& Training [s] & $8.0$ & $21.4$ & $241.9$ \\
\midrule
\multirow{4}{*}{FFOLayer (Lifted)}
& Stability margin $t^\star$ & $8.3\!\times\!10^{-3}$ & $6.3\!\times\!10^{-3}$ & $3.4\!\times\!10^{-3}$ \\
& Forward [ms] & $28.1$ & $84.3$ & $797.4$ \\
& Backward [ms] & $30.5$ & $88.9$ & $815.1$ \\
& Training [s] & $15.1$ & $44.2$ & $405.9$ \\
\midrule
\multirow{4}{*}{CVXPYLayers}
& Stability margin $t^\star$ & $8.5\!\times\!10^{-3}$ & $6.3\!\times\!10^{-3}$ & $3.4\!\times\!10^{-3}$ \\
& Forward [ms] & $93.3$ & $345.3$ & $5592.1$ \\
& Backward [ms] & $6.6$ & $84.3$ & $4948.6$ \\
& Training [s] & $28.6$ & $115.2$ & $13047.9$ \\
\bottomrule
\end{tabular*}
}
\par
\endgroup

\end{table}

After epoch $50$, we harden the switching sequence to
$\sigma(k)=\mathbb{I}[k\geq\tau]$, freeze $\tau$, and optimize only $A_0,A_1$
for another $200$ epochs. Rather than recursively feeding back $\widehat x_k$,
the hard-stage one-step loss predicts every next state directly from the
corresponding observed current state:
\begin{equation}
    \mathcal{L}_{\mathrm{1step}}
    =\frac{1}{N_{\mathrm{traj}}(T-1)}
    \sum_{j,k}\left\|x_{k+1}^{(j)}
    -A_{\sigma(k)}x_k^{(j)}-B_{\sigma(k)}u_k^{(j)}\right\|_2^2.
\end{equation}
This teacher-forced objective isolates transition identification from the
compounding error of a long recursive rollout. The hard stage uses Adam with
learning rate $3\times10^{-3}$. Thus the discontinuity near epoch $50$ in
Figure~\ref{fig:learning_experiments} (Right) is expected: it coincides with both the
soft-to-hard change in the switching model and the change from
$\mathcal{L}_{\mathrm{roll}}$ to $\mathcal{L}_{\mathrm{1step}}$.

\section{Gradient Accuracy}
\label{app:gradient_accuracy}

We validate gradient accuracy on the random SOCP and SDP instances described in
Section~\ref{sec:synthetic_experiments} and Appendix~\ref{app:experiment_details}.  We first compare the gradients produced
by dOPT and CVXPYLayers.  For each method, we vectorize the gradients with
respect to all problem parameters and concatenate them into a single vector.
We report the relative error
\begin{equation}
    \frac{\lVert g_{\mathrm{candidate}}-g_{\mathrm{reference}}\rVert_2}
    {\max\{\lVert g_{\mathrm{reference}}\rVert_2,10^{-12}\}}.
\end{equation}
In each ``A vs. B'' comparison, A is the candidate and B is the reference.

We additionally verify the gradient using the envelope theorem
\citep{afriat1971theory}.  Define the value function as
$V(\theta) := f_\theta(z^*(\theta))$, and let
$\mathcal{L}_\theta(z,\mu,\nu)$ denote the Lagrangian.  The theorem establishes
an equivalence between the derivative of the value function and the partial
derivative of the Lagrangian evaluated at the optimal primal--dual solution:
\begin{equation}
    \nabla_\theta V(\theta)
    =
    \left.
    \partial_\theta \mathcal{L}_\theta(z,\mu,\nu)
    \right|_{(z,\mu,\nu)
    =(z^*(\theta),\mu^*(\theta),\nu^*(\theta))}.
\end{equation}
Here, $\partial_\theta$ differentiates only the explicit dependence on
$\theta$, holding $(z,\mu,\nu)$ fixed.  The right-hand side is therefore an
analytic gradient, which we use as ground truth (GT).  Each entry below is the
mean relative error over ten random instances, using the same forward solution
across methods.
\begin{table}[htbp]
    \centering
    \caption{Gradient relative errors on random SOCPs.}
    \label{tab:socp_gradient_accuracy}
    \small
    \begin{tabular}{lccc}
        \toprule
        Gradient comparison & $n=20$ & $n=100$ & $n=500$ \\
        \midrule
        $\nabla_\theta \mathbf{1}^\top z^*(\theta)$: dOPT vs. CVXPYLayers
            & $1.38\times10^{-2}$ & $3.24\times10^{-3}$ & $1.60\times10^{-1}$ \\
        $\nabla_\theta V(\theta)$: dOPT vs. GT
            & $3.67\times10^{-9}$ & $5.81\times10^{-9}$ & $4.10\times10^{-9}$ \\
        $\nabla_\theta V(\theta)$: CVXPYLayers vs. GT
            & $6.88\times10^{-4}$ & $2.55\times10^{-4}$ & $1.50\times10^{-3}$ \\
        \bottomrule
    \end{tabular}
\end{table}
\begin{table}[htbp]
    \centering
    \caption{Gradient relative errors on random SDPs.}
    \label{tab:sdp_gradient_accuracy}
    \small
    \begin{tabular}{lccc}
        \toprule
        Gradient comparison & $n=20$ & $n=50$ & $n=100$ \\
        \midrule
        $\nabla_\theta \mathbf{1}^\top z^*(\theta)$: dOPT vs. CVXPYLayers
            & $2.24\times10^{-2}$ & $8.21\times10^{-3}$ & $2.49\times10^{-2}$ \\
        $\nabla_\theta V(\theta)$: dOPT vs. GT
            & $1.09\times10^{-7}$ & $1.22\times10^{-7}$ & $1.03\times10^{-7}$ \\
        $\nabla_\theta V(\theta)$: CVXPYLayers vs. GT
            & $5.89\times10^{-3}$ & $4.51\times10^{-3}$ & $1.11\times10^{-2}$ \\
        \bottomrule
    \end{tabular}
\end{table}

The discrepancy between dOPT and CVXPYLayers remains modest, although it is
larger at the highest tested dimension.  For
$\nabla_\theta V(\theta)$, dOPT agrees consistently with the analytic
GT, while the CVXPYLayers error becomes more pronounced at larger dimensions.

\end{document}